\documentclass[letterpaper]{article} 
\usepackage{aaai2026}  
\usepackage{times}  
\usepackage{helvet}  
\usepackage{courier}  
\usepackage[hyphens]{url}  
\usepackage{graphicx} 
\usepackage{natbib}  
\usepackage{caption} 
\usepackage{mathtools}
\usepackage{amsmath}
\usepackage{amssymb}
\usepackage{amsthm}          
\newtheorem{theorem}{Theorem}
\newtheorem{lemma}[theorem]{Lemma}
\usepackage{algorithm}
\usepackage[noend]{algpseudocode}
\usepackage{graphicx}    
\usepackage{subcaption}  
\usepackage{caption}     

\usepackage{newfloat}
\usepackage{listings}
\DeclareCaptionStyle{ruled}{labelfont=normalfont,labelsep=colon,strut=off} 
\floatstyle{ruled}
\newfloat{listing}{tb}{lst}{}
\floatname{listing}{Listing}
\usepackage{pifont}

\title{Distributed Algorithms for \\Multi-Agent Multi-Armed Bandits with Collision}
\author{
  Daoyuan Zhou\textsuperscript{\rm 1},
  Xuchuang Wang\textsuperscript{\rm 2},
  Lin Yang\textsuperscript{\rm 1},
  Yang Gao\textsuperscript{\rm 1}
}
\affiliations{
  \textsuperscript{\rm 1}School of Artificial Intelligence, Nanjing University, China\\
  \textsuperscript{\rm 2}College of Information and Computer Sciences, University of Massachusetts Amherst, USA\\[4pt]
  \texttt{221900247@smail.nju.edu.cn, xuchuangw@gmail.com,  linyang@nju.edu.cn, yanggao@nju.edu.cn}
}

\begin{document}

\maketitle

\begin{abstract}
We study the stochastic Multiplayer Multi-Armed Bandit (MMAB) problem, where multiple players select arms to maximize their cumulative rewards. Collisions occur when two or more players select the same arm, resulting in no reward, and are observed by the players involved. We consider a distributed setting without central coordination, where each player can only observe their own actions and collision feedback. 
We propose a distributed algorithm with an adaptive, efficient communication protocol. The algorithm achieves near-optimal group and individual regret, with a communication cost of only  $\mathcal{O}(\log\log T)$. Our experiments demonstrate significant performance improvements over existing baselines. Compared to state-of-the-art (SOTA) methods, our approach achieves a notable reduction in individual regret. Finally, we extend our approach to a periodic asynchronous setting, proving the lower bound for this problem and presenting an algorithm that achieves logarithmic regret.
\end{abstract}

\maketitle
\section{Introduction}
Multi-agent multi-armed bandit (MMAB) is a fundamental extension of the canonical multi-armed bandit model~\citep{lai1985asymptotically} in sequential decision-making. A standard MMAB setting consists of $K \in \mathbb{N}^+$ arms and $M \in \mathbb{N}^+$ agents, where each arm yields rewards drawn from an unknown distribution.
Much of the existing literature on MMAB focuses on simplified settings in which collisions are ignored~\citep{szorenyi2013gossip,landgren2016distributed,chakraborty2017coordinated,kolla2018collaborative,martinez2019decentralized,feraud2019decentralized,wang2020optimal,bistritz2020cooperative,chawla2020gossiping,yang2021cooperative,chen2023demand}. That is, simultaneous pulls on the same arm yield independent rewards without interference. Under this assumption, the common goal is to maximize the total cumulative reward over a long time horizon $T \in \mathbb{N}^+$.
However, collisions are inevitable in real-world applications, and their effects cannot be ignored. In practice, simultaneous access to shared resources often leads to interference, degrading performance, or causing failure. For example, in cognitive radio (CR) networks~\citep{jouini2009multi}, devices competing for spectrum experience collisions that yield no reward, underscoring the need to explicitly model them.
In addition to collisions, individual regret becomes a crucial metric in decentralized systems. In decentralized systems like CR networks, each device acts independently without coordination. As a result, minimizing individual regret—rather than maximizing collective reward—becomes the primary and more realistic objective.

To improve cooperation efficiency under decentralized constraints, this paper studies the Collision-Sensing Multi-Agent Multi-Armed Bandits (CS-MMAB) model~\citep{besson2018multi}, where agents operate in a shared environment with access to collision feedback. In CS-MMAB, simultaneous pulls on the same arm by multiple agents result in observable collisions, which are the primary source of regret. Therefore, effective collision avoidance is critical for improving overall performance. Instead of simply avoiding collisions, some prior works~\citep{boursier2019sic,wang2020optimal,shi2021heterogeneous} exploit collisions as a medium for information exchange, using bit-encoding techniques to embed messages within collision patterns. This enables agents to transmit partial estimates and arm identities over time, even without explicit communication channels, facilitating collaborative arm identification and knowledge sharing among agents, and thus reducing redundant exploration. We further adopt this model in a distributed setting, where all agents participate equally in cooperation (communication) and learning (arm exploration). Such cooperation suits homogeneous systems like cooperative robotics or wireless networks. The distributed framework allows each agent to collaboratively explore and refine its decision-making to minimize individual regret, a critical performance metric in resource-constrained multi-agent systems.
\begin{table*}[t]
\small
\caption{Comparison of MMAB algorithms under collision settings. }
\centering
\begin{tabular}{|c|c|c|c|c|}
\hline
\textbf{Algo.} & \textbf{Ind. Regret} & \textbf{Grp. Regret} & \textbf{Comm. Regret} \\
\hline
\texttt{DPE1} \cite{wang2020optimal}    & $O(K\log T)$     & $O(K\log T)$     & $O(cK^{5/2}M^3\Delta^{-2})$ \\
\texttt{\texttt{SIC-MMAB}} \cite{boursier2019sic}   & ---                      & $O(K\log T)$     & $O(KM^3 \log^2(\frac{\log T}{\Delta^2}))$ \\
\texttt{SynCD(ours)}     & $O(\frac{K}{M}\log T)$ & $O(K\log T)$     & $O\left(M^3 \left(K \log\left(\frac{1}{\Delta_{\text{min}}}\right)+ \log\log T \right)\right)$ \\
\hline
\end{tabular}

\label{tab:mmab-comparison}
\end{table*}    
Existing works on decentralized Multi-Agent Multi-Armed Bandits (MMAB) with collisions have largely focused on balancing regret minimization and communication cost. A prominent line of work adopts leader-follower frameworks, such as \texttt{DPE1}~\citep{wang2020optimal}, which attain optimal group regret with limited communication. However, these approaches raise fairness concerns, as the leader bears the full learning burden, leading to imbalanced individual regret. Distributed algorithms like \texttt{SIC-MMAB}~\citep{boursier2019sic} address this issue by symmetrically distributing learning across agents. Yet, they still incur high individual regret. Once optimal arms are identified, agents tend to behave greedily and monopolize them, forcing others into suboptimal choices. These algorithms also require frequent communication, which is often impractical. 

Moreover, most existing methods rely on strict synchronization, making them unsuitable for asynchronous settings where agent activity varies over time. In asynchronous MMAB (AMMAB), lack of synchrony complicates coordination and limits the applicability of synchronized methods.

We summarize the key contributions of this paper as follows:

\textbf{(i)} We propose \texttt{SynCD}, a synchronized, distributed algorithm for the collision bandit setting. \texttt{SynCD} achieves near-optimal group regret 
$O\left(\sum_{k > M} \frac{\log T}{\Delta_k} \right),$
near-optimal individual regret 
$O\left( \frac{1}{M} \sum_{k > M} \frac{\log T}{\Delta_k} \right),$
and an extremely low communication cost of \( \mathcal{O}(\log \log T) \) bits. While this is not strictly constant, it remains asymptotically negligible and substantially improves over prior work \cite{boursier2019sic} in both efficiency and scalability. Our algorithmic and theoretical guarantees are summarized in Table~\ref{tab:mmab-comparison}, highlighting key improvements over existing methods.

\textbf{(ii)} We extend our framework to an asynchronous setting where agents have periodic activations. In this setting, we design a Distributed algorithm that adapts to periodic agent schedules and establishes the first known regret lower bound for this variant of asynchronous MMAB.

\textbf{(iii)} We perform extensive simulations comparing \texttt{SynCD} with existing baselines, including \texttt{SIC-MMAB}\citep{boursier2019sic} and \texttt{DPE1}\citep{wang2020optimal}. Results show that \texttt{SynCD} consistently and significantly outperforms \texttt{SIC-MMAB} in all three aspects: group regret, individual regret, and communication efficiency. Moreover, \texttt{SynCD} achieves substantially lower individual regret than \texttt{DPE1}, establishing its empirical superiority under synchronized settings.

\section{Related Work}
\label{sec:related_work}

Many works~\citep{szorenyi2013gossip,landgren2016distributed,chakraborty2017coordinated,kolla2018collaborative,martinez2019decentralized,wang2020optimal,bistritz2020cooperative,yang2021cooperative,chen2023demand} have assumed centralized control or idealized \emph{collision-free} settings. However, collisions are inevitable in practice, and agents must coordinate to minimize interference and avoid suboptimal performance.

To address coordination through collision feedback, early methods in homogeneous settings~\citep{kumar2010early,liu2010distributed} relied on pre-agreements among agents, using fixed rankings to sequentially select arms and avoid collisions. Later, randomized strategies such as Rand Orthogonalization and Musical Chairs introduced probabilistic exploration. The Musical Chairs protocol was further combined with Explore-Then-Commit (ETC) strategies to stabilize collision-free allocations~\citep{nayyar2016regret}. Orthogonalization was also embedded into the UCB framework to balance exploration and exploitation~\citep{Besson2018}.

In heterogeneous environments, achieving optimal matching proved more challenging. Early algorithms~\citep{Avner2019,Darak2019} only guaranteed Pareto efficiency, which motivated the development of communication-based approaches. For example, a leader-follower framework was proposed in~\citep{Kalathil2014}, where a central agent coordinated exploration via real-valued data exchange, thereby reducing communication but introducing potential bottlenecks. In contrast, a distributed Markov-chain mechanism, known as “Game of Thrones,” was introduced in~\citep{Bistritz2020}, enabling agents to dynamically adjust matchings without centralized control.

A significant advancement was achieved by the SIC-MMAB algorithm~\citep{boursier2019sic}, which repurposed collisions as communication signals, allowing agents to share quantized mean estimates and attain near-centralized performance. Building on this idea, DPE1~\citep{wang2020optimal} adopted a hybrid leader-follower design that minimized communication while maintaining asymptotic optimality. In heterogeneous environments, BEACON~\citep{Shi2021} further improved efficiency by combining batched learning with differential signaling, using a leader to synchronize epochs and approximate centralized performance with limited communication.

Recent efforts have moved beyond strong synchronization assumptions towards more realistic, temporally uncoordinated models. These include asynchronous settings~\citep{bonnefoi2017multi,dakdouk2021multi,richard2023asynchronous} where agent actions are intermittent, and dynamic models~\citep{avner2014concurrent,rosenski2016multi,bande2019collaborative,boursier2019sic,darak2019multi} that allow for agent arrival and departure. Although progress has been made in centralized asynchronous settings~\citep{richard2023asynchronous}, achieving robust guarantees in decentralized, asynchronous, and dynamic environments remains a key open challenge.

\section{Model}
\label{sec:model}
We consider a \textit{stochastic multi-agent multi-armed bandit} (MMAB) problem with a set of agents \( \mathcal{M} \coloneqq \{1, 2, \ldots, M\} \) and a set of arms \( \mathcal{K} \coloneqq  \{1, 2, \ldots, K\} \), where \( M < K \). The agents are active over a time horizon of \( T \) rounds.
At each round \( t \in \mathcal{T} := \{1, 2, \ldots, T\} \), each agent \( m \in \mathcal{M} \) independently selects an arm \( k \in \mathcal{K} \) to pull. Each arm \( k \) yields an i.i.d. reward sequence \( (X_t^{(m)}(k))_{t \in [T]} \) to each agent \( m \), where \( X_t^{(m)}k) \in [0, 1] \) and \( \mathbb{E}[X_t^{(m)}(k)] = \mu(k) \). For the sake of simplicity, we assume that the arm means are strictly ordered as $\mu(1)  > \mu(2) > \cdots > \mu(K).$
This assumption does not have any impact on the analysis or the problem formulation.
The instantaneous reward received by agent \( m \) at time \( t \) is defined as
\[
r_t^{(m)}(k) := X_t^{(m)}(k) \cdot (1 - \eta_t(k)),
\]
where \( \eta_t(k) = 1 \) if arm \( k \) is selected by multiple agents at time \( t \), and \( \eta_t(k) = 0 \) otherwise. That is, collisions lead to zero reward for all agents pulling the same arm.

The goal is to minimize the cumulative \textit{pseudo-regret}, defined as
\[
\mathbb{E}[R(T)] = T \sum_{k=1}^{M} \mu(k) - \mathbb{E}_{\mu}\left[\sum_{t=1}^{T} \sum_{m \in \mathcal{M}} r^{(m)}_t\right],
\]
which measures the gap between the optimal assignment of the top-\( M \) arms and the actual rewards obtained.

To capture fairness, we define the individual cumulative regret as
\[
R_{\text{ind}}(T) := \max_{m \in \mathcal{M}} \mathbb{E}_\mu \left[ T \bar{\mu}^* - \sum_{t=1}^T r_t^{(m)} \right],
\]
where \(\bar{\mu}^*\) is the average of the top \(M\) arm means:
$
\bar{\mu}^* = \frac{1}{M} \sum_{k=1}^M \mu(k).
$
For detailed regret analysis, we then decompose the regret into three parts,
\begin{align}
\mathbb{E}[R(T)] = 
R^{\mathrm{init}} + R^{\mathrm{comm}} + R^{\mathrm{explo}}.
\end{align}
\begin{align}
\small
\text{where} \quad
\left\{
\begin{aligned}
R^{\mathrm{init}} &= T_{\mathrm{init}} \sum_{k=1}^{M} \mu(k)
- \mathbb{E}_{\mu} \left[\sum_{t=1}^{T_{\mathrm{init}}} \sum_{j=1}^{M} r^{(j)}_t\right], \\
R^{\mathrm{comm}} &= \mathbb{E}_{\mu} \left[ \sum_{t \in \mathrm{Comm}} \sum_{j=1}^{M} 
\left( \mu(j) - r^{(j)}_t \right) \right], \\
&\text{with Comm the set of communication steps}, \\
R^{\mathrm{explo}} &= \mathbb{E}_{\mu} \left[ \sum_{t \in \mathrm{Explo}} \sum_{j=1}^{M} 
\left( \mu(j) - r^{(j)}_t \right) \right], \\
&\text{with } \mathrm{Explo} = \{ T_{\mathrm{init}} + 1, \ldots, T \} \setminus \mathrm{Comm}.
\end{aligned}
\right.
\end{align}

\section{Algorithms}
\label{sec:synalgo}
In this section, we introduce a Synchronous Communication-Efficient Distributed (\texttt{SynCD}) algorithm that achieves low individual regret and low communication overhead. In synchronized, collision-sensing environments, \textit{forced collisions} can be utilized to transmit messages, enabling centralized-optimal regret. However, prior protocols often incur substantial communication costs, primarily due to redundant communication rounds and large message payloads.

Our key improvement lies in significantly reducing the communication volume per round, particularly under sparse communication regimes. This is achieved via a threshold-based triggering mechanism~\citep{yang2023cooperative}, where each agent maintains an Estimated Confidence Radius (ECR) and initiates communication only when its ECR drops below a fraction \(\beta\) of its previous value. This ensures that communication occurs only when necessary. To mitigate performance degradation due to infrequent communication, agents eliminate suboptimal arms locally using both their own statistics and the partial updates received from others. Unlike \texttt{SIC-MMAB}~\citep{boursier2019sic}, which transmits full arm statistics, we adopt an adaptive differential communication strategy~\citep{shi2021heterogeneous}. Quantization precision is dynamically adjusted based on global estimation accuracy, ensuring consistency while minimizing message size.
This adaptive scheme leads to a low total communication cost and significantly improves efficiency over previous methods.

Moreover, leader-follower strategies or greedy approaches like \texttt{SIC-MMAB} and DPE1 often incur high individual regret due to imbalanced arm usage. In contrast, our algorithm combines arm elimination with uniform exploitation over the set of optimal arms. This ensures fair and collision-free access to top arms, thereby achieving improved individual regret.

\begin{algorithm}[t]
\small
\caption{Distributed Bandit Algorithm for Agent $m$}
\label{alg:SynCD}
\begin{algorithmic}[1]
\State \textbf{Input:} Time horizon $T$, exploration slack $\beta > 1$, arm set $\mathcal{K}$ with $|\mathcal{K}| = K$
\Comment{Initialization}
\State Initialize $T_0 = \lceil K \log T \rceil$, $\mathcal{K}_t = \mathcal{K}$, $K_t = K$
\State Initialize ECR=1, $T_t(k)$, $n_t^{(m)}(k)$, $\tilde{\mu}_t^{(m)}(k) = 0$, $\forall k$
\State $\text{Acc} \gets \emptyset$, $\text{Rej} \gets \emptyset$
\State Initialization $j \gets \textsc{InitPhase}(M,\mathcal{K})$
\Comment{Orthogonalization + Rank Assignment}
\While{$|\text{Acc}| < M$ and $t < T$}
    \Comment{Exploration}
    \If{Acc or Rej changed and no communication}
        \State Pull uniformly for $K_t \cdot M$ rounds
    \Else
        \For{cycle $= 1$ to $K_t$}
            \For{$p = 0$ to $M - 1$}
                \State $t \gets t + 1$ 
                \State \textbf{if} $p \in$ \textsc{ExploitSlot}$(j)$ 
                \textbf{then} $k \gets \text{Acc}[(p - j) \bmod |\text{Acc}|]$ 
                \textbf{else} $k \gets \mathcal{K}_t[(p - (j + M - M_t)) \bmod K_t]$
                \State Pull arm $k$
            \EndFor
        \EndFor
    \EndIf
    \If{Collision occurred}   \Comment{Communication}
        \State \textsc{Comm\_a}$(\text{Rej}, \text{Acc})$ and update $t$
        \State Update $\text{Rej}$ and $\text{Acc}$
    \Else
        \State $n_t^{(m)}(k) \gets n_t^{(m)}(k) + M_t$
        \State $T_t(k) \gets T_t(k) + K_t \cdot M_t$
        \If{$\text{ECR}_t \le \beta \cdot \text{ECR}_{\text{last}}$}
            \State Update $\text{ECR}_{\text{last}}$, $X^{(m)}_{\text{last}}(k)$, $T_{\text{last}}(k)$, $n^{(m)}_{\text{last}}(k)$
            \State \textsc{Comm}$(\delta_t^{(m)}(k))$ and update
            \State Update $\tilde{\mu}^{(m)}(k)$
        \EndIf
        \State Update $\hat{\mu}_t^{(m)}(k)$
        \State Update $\text{Rej}$ and $\text{Acc}$ using Rules (\ref{acc}) and (\ref{rej})
        \EndIf

    \State $\mathcal{K}_t \gets \mathcal{K}_t \setminus (\text{Rej} \cup \text{Acc})$
\EndWhile
\State Agents take turns pulling arms in $\text{Acc}$ via round-robin until $t = T$ \Comment{Exploitation}
\end{algorithmic}
\end{algorithm}
\subsection{Elimination-Based Learning Policy}
To enable effective decentralized learning, an appropriate initialization is necessary to ensure coordination among agents. Before exploration, we adopt the initialization procedure $\textsc{InitPhase}(M,\mathcal{K})$ from \texttt{DPE1}~\citep{wang2020optimal}, which consists of an \textit{orthogonalization phase} followed by a \textit{rank assignment phase}, assigning each agent a unique index $j \in \{0, \dots, M{-}1\}$. This symmetry breaking is essential for structured communication and collaborative learning in a distributed setting. See Appendix for detailed implementation.

Consider \(M\) agents and an active arm set \(\mathcal{K}_t\) at time \(t\), where \(M_t\) agents explore and \(K_t = |\mathcal{K}_t|\) arms are active. Each phase contains \(M K_t\) steps organized into \(K_t\) cycles of length \(M\). In each cycle, every agent exploits for \(M-M_t \) steps and explores for \(M_t\) steps, scheduled cyclically.

The exploitation time slots for agent \(j\), denoted \(\textbf{ExploitSlots}(j)\), are defined as follows: 
\[
\small
\begin{cases}
[j,\ j + M - M_t) \quad \quad \quad \text{if } j + M - M_t \leq M \\
[j,\ M) \cup [0,\ (j + M - M_t) \bmod M) & \text{otherwise}
\end{cases}
\]
During these slots, agent \(j\) pulls arms from the accepted set \(\text{Acc}\) in a round-robin fashion; during other slots, it pulls uniformly from \(\mathcal{K}_t\). This schedule ensures uniform and collision-free access to both accepted and candidate arms.

After identifying the top arms, agents enter a pure exploitation phase, cyclically pulling from the optimal set using the same collision-avoidance schedule to ensure low individual regret.

\paragraph{Elimination and Acceptance.} 
A key goal in our algorithm is to \emph{accept optimal arms} and \emph{reject suboptimal ones} as early as possible. 
These occur in two stages: (1) immediately after communication rounds, when agents jointly update arm statistics and eliminate arms based on global information; (2) locally at the end of each exploration phase, using each agent’s private observations.

Let $\hat{\mu}_{t}^{(m)}(k)$ denote agent $m$’s estimate of the mean reward of arm $k$ at time $t$, and let $\tilde{\mu}^{(m')}(k)$ denote the most recently reconstructed estimate of agent $m'$ for arm $k$, which is updated by cumulatively adding the reward differences transmitted during each communication round.
Let $n_t^{(m')}(k)$ be the total number of times agent $m'$ has pulled arm $k$ up to time $t$, and $n_{\text{last}}^{(m')}(k)$ be the number of pulls prior to the last synchronization. Similarly, let $X^{(m')}_t(k)$ and $X_{\text{last}}^{(m')}(k)$ denote the total accumulated rewards for arm $k$ at the current time and at the last synchronization, respectively.
Define $T_{\text{last}}^{(m')}(k)$ as the estimated number of total pulls (across agents) on arm $k$ prior to the last communication. Then, the updated estimator is given by:
\[
\small
\hat{\mu}_{t}^{(m)}(k) = \frac{\displaystyle \sum_{m'=1}^{M} n_{\text{last}}^{(m')}(k)\,\tilde{\mu}^{(m')}(k) + X^{(m)}_t(k) - X_{\text{last}}^{(m)}(k)}{T_{\text{last}}(k) + n_t^{(m)}(k) - n_{\text{last}}^{(m)}(k)}
\]

Let $N_t(k)$ be the sample count for arm $k$, and define the confidence radius as $R(k,t) = 2\beta\sqrt{\frac{\log(1/\delta)}{2N_t(k)}}$, with confidence level $1 - \delta$.
We define the lower and upper confidence bounds of arm \( k \) for agent \( m \) at time \( t \) as
\(
\mathrm{LCB}_t^{(m)}(k) := \hat{\mu}_t^{(m)}(k) - R(k,t),\quad
\mathrm{UCB}_t^{(m)}(k) := \hat{\mu}_t^{(m)}(k) + R(k,t).
\)

Arm \( k \) is \textbf{accepted} if
\begin{align}
    \label{acc}
    \left| \left\{ i \in \mathcal{K}_t \;\middle|\; \mathrm{LCB}_t^{(m)}(k) \ge \mathrm{UCB}_t^{(m)}(i) \right\} \right| \ge K_t - M_t
\end{align}

Arm \(k \) is \textbf{eliminated} if
\begin{align}
    \label{rej}
    \left| \left\{ i \in \mathcal{K}_t \;\middle|\; \mathrm{UCB}_t^{(m)}(k) \le \mathrm{LCB}_t^{(m)}(i) \right\} \right| \ge M_t.
\end{align}

If any arm is accepted or eliminated and no communication is scheduled, an extra round is triggered to synchronize the active arm sets.

\subsection{Adaptive Communication Protocol}
/Our approach combines a threshold-based triggering mechanism with adaptive quantization and truncation, ensuring accurate synchronization of arm statistics while keeping communication overhead low.

\textbf{Tiggering-Based Mechanism.} To manage communication efficiently, we adopt a threshold-based trigger mechanism inspired by~\citet{yang2023cooperative} to avoid unnecessary communication. Each agent tracks an \textit{Estimated Confidence Radius} $\text{ECR}(k) = \sqrt{\frac{\log(1/\delta)}{2T_t(k)}},
$ for each active arm \(k\)
where \(T_t(k)\) is the number of pulls of arm \(k\) up to time \(t\). Since agents explore arms uniformly and synchronously, their ECRs remain aligned. Communication is triggered when \(\text{ECR}(k)\) decreases by a factor of \(\beta\) from its value at the previous communication, ensuring communication only occurs when local estimates may have diverged significantly. Unlike the conventional doubling trick that initializes ECR at zero, we set it to 1, guaranteeing a constant number of communication rounds but delaying the first communication to roughly 
$O(\log T)$ pulls.

\textbf{Adaptive Quantization and Truncation.}  
To minimize communication volume, agents do not transmit raw statistics. Instead, they communicate the quantized difference between empirical means computed at two successive communication rounds. For arm \(k\) and agent \(m\), the empirical mean \(\bar{\mu}_t^{(m)}(k)\) is quantized using:
$\tilde{\mu}^{(m)}(k) = \text{Ceil}\left(\bar{\mu}_t^{(m)}(k)\right)$with$  \left\lceil 1 + \frac{\log T_t(k)}{2} \right\rceil$ bits,
where the precision increases with \(T_t(k)\), reflecting improved estimate accuracy over time. Rather than transmitting absolute values, agents compute the difference \(\tilde{\delta}_k^{(m)} = \tilde{\mu}_k^{(m)} - \tilde{\mu}_{k,\text{last}}^{(m)}\), and only transmit its most significant non-zero bits. For instance, the binary string \texttt{000110} is truncated to \texttt{110}, reducing redundancy significantly.

\textbf{Communication Protocol.}
Agents communicate using a forced-collision signaling mechanism, where binary messages are encoded through intentional collisions on a designated arm: pulling the arm signals a bit 1, while staying idle signals a bit 0. To support variable-length messages, communication follows a signal-then-transmit routine and ends when a signal collision is detected. The protocol is summarized in Algorithm~\ref{alg:comm}.
\begin{algorithm}[htbp]
\small
\caption{\textsc{Comm}: Main Communication Protocol}
\label{alg:comm}
\textbf{Input:} $\delta_t^{(m)}(k)$ (local statistic), $j$ (internal rank) \\
\textbf{Output:} $\tilde{\mathbf{S}}$ (quantized statistics of all active players)
\begin{algorithmic}[1]
    \State $\tilde{\delta}_t^{(m)}(k) \gets \texttt{Quantization}(\delta_t^{(m)}(k))$
    \State $E_t := \{(i, l, k) \in [M] \times [M] \times \mathcal{K}_t \mid i \neq l\}$
    \For{$(i, l, k) \in E_t$}
        \If {$i = j$}
            \State \texttt{Send}($j, \tilde{\delta}_t^{(m)}(k), l$)
        \ElsIf{$l = j$}
            \State $\tilde{\mathbf{S}} \gets \texttt{Receive}(l)$
        \Else
            \State \texttt{Wait}($j, l$)
        \EndIf
    \EndFor
\end{algorithmic}
\end{algorithm}

Also, communication may be triggered when structural changes occur in the arm set, such as the acceptance or elimination of an arm. To initiate communication, the sender pulls a designated arm during its exploration window. If another agent simultaneously explores the same arm, a collision occurs, triggering a synchronization protocol. Upon detecting the collision, all agents pause until the current exploration ends, after which they collectively process the received updates. This procedure ensures that all agents maintain a consistent view of the active arm set before entering the next phase, thereby preserving the coherence and correctness of distributed decisions.
See details in algorithm ~\ref{comm_a}.
\begin{algorithm}[htbp]
\small
\caption{\textsc{Comm\_a}: Arm Synchronization via Collision}
\label{comm_a}
\textbf{Input:} \texttt{Acc}, \texttt{Rej}, $j$ (own internal rank) \\
\textbf{Output:} Updated \texttt{Acc}, \texttt{Rej}
\begin{algorithmic}[1]
\State $E_t := \{(i, l) \in [M] \times [M] \mid i \neq l\}$
\For{$(i, l) \in E_t$}
    \For{$\texttt{mode} \in \{\texttt{Acc}, \texttt{Rej}\}$}
        \If{$i = j$} \Comment{Sender}
            \For{$k \in \mathcal{K}_t$}
               \State Pull $l$ if $k$ newly marked in \texttt{mode}, else $j$
            \EndFor
        \ElsIf{$l = j$} \Comment{Receiver}
            \For{$k \in \mathcal{K}_t$}
                \State Pull arm $j$
                \If{collision detected}
                    \State Update \texttt{mode} $\gets k$

                \EndIf
            \EndFor
        \Else \Comment{Waiter}
            \For{$k \in \mathcal{K}_t$}
                \State Pull arm $j$
            \EndFor
        \EndIf
    \EndFor
\EndFor
\end{algorithmic}
\end{algorithm}

\subsection{Theoretical Analysis}
In this section, we present theoretical guarantees for the proposed learning and communication framework. We provide upper bounds on group and individual regret, as well as the total communication cost, demonstrating that our algorithm is both learning-efficient and communication-efficient in the distributed multi-agent setting. The main result is summarized in the theorem below.

\begin{theorem}
\label{thm:bandit}
Set $\delta = 1/T^2$, $\beta > 1$. Then Algorithm~\ref{alg:SynCD} achieves the following performance guarantees, under appropriately calibrated confidence intervals and communication control. 
{\small
\begin{itemize}
\item[(i)] \textbf{Group regret:}
\[
\begin{aligned}
&\mathbb{E}[R(T)] \leq
\sum_{k > M}
\left[
\frac{32\beta^2(\beta + 2)\log T}{\Delta(k)} +
MK\Delta(k)
\right]+C
\\
&+ 2M^3 K + 2M^3 \cdot \left(
1 + \frac{1}{2} \log \left( \frac{\log T}{ \beta^{2}} + M K\right)
\right)+\\
&
2M^3 \left(\sum_{k=1}^K \log_{\beta}\left(\frac{8 \beta}{\Delta(k)}\right)-1\right)
\left[
7 + \log_2\left(1 + \beta + \sqrt{\frac{M \ln 2}{2}}\right)
\right]\\
\end{aligned}
\]

\item[(ii)] \textbf{Individual regret:}
\[
\begin{aligned}
&\mathbb{E}[R_{\mathrm{ind}}(T)] \leq
\frac{1}{M}\Bigg(\sum_{k > M}
\left[
\frac{32\beta^2(\beta + 2)\log T}{\Delta(k)} +
MK\Delta(k)
\right]+C
\\
&+ 2M^3 K + 2M^3 \cdot \left(
1 + \frac{1}{2} \log \left( \frac{\log T}{ \beta^{2}} + M K\right)
\right)+\\
&
2M^3 \left(\sum_{k=1}^K \log_{\beta}\left(\frac{8 \beta}{\Delta(k)}\right)-1\right)
\left[
7 + \log_2\left(1 + \beta + \sqrt{\frac{M \ln 2}{2}}\right)
\right]\Bigg)\\
\end{aligned}
\]
\end{itemize}
}
\noindent where $\Delta(k) = \mu(k) - \mu(M+1)$ if $k \le M$ and $\Delta(k) = \mu(M) - \mu(k)$ if $k > M$ and C denotes the constant regret during initialization.
\end{theorem}
\noindent\textbf{Proof Sketch of Theorem~\ref{thm:bandit}.}  
We compute the total regret by decomposing it into three parts, as introduced in Section~\ref{sec:model}.
Following our initialization procedure, we incur only a constant regret.  
Next, we construct confidence intervals for the mean rewards of each arm.  
By applying this to standard multi-armed bandit analysis, we derive an upper bound on the number of pulls for each suboptimal arm.  
Then, using standard techniques such as telescoping sums and appropriate rescaling, we obtain the exploration regret bound.For the regret from communication, we first show that the number of communication rounds is constant. The initial communication round incurs a cost of $\mathcal{O}(\log\log T)$ bits due to full quantization. Subsequent updates require only a constant number of bits using differential quantization.  
Summing across all communication rounds yields the total communication regret.

\noindent\textbf{Remark.} Theorem \ref{thm:bandit} establishes that our algorithm is optimal with respect to regret, simultaneously achieving near-optimal group regret and optimal individual regret. This result matches the information-theoretic lower bound, representing a significant theoretical contribution.

The primary trade-off for this strong guarantee lies in the communication complexity. Our analysis reveals a nuanced scaling behavior. On one hand, the communication cost is remarkably efficient with respect to the time horizon, growing at an exceptionally slow rate of $O(\log\log T)$. On the other hand, the cost scales polynomially with the number of agents and arms as $O(KM^3)$, which can present a scalability bottleneck in systems with large $M$ or $K$. This polynomial overhead is a direct consequence of the forced-collision protocol essential for coordination in a decentralized setting.

\section{Extending to Asynchronous Settings}
Decentralized systems often operate asynchronously, where agents act at irregular intervals. This presents significant challenges for multi-agent bandits, including inconsistent observations, a lack of synchronization, and stale information. Our goal is to extend the efficient \texttt{SynCD} algorithm to these asynchronous environments, preserving its communication efficiency and regret guarantees while operating under minimal structural assumptions.
\subsection{Asynchronous Setups}\label{sec: asych}
We now consider the same MMAB problem under asynchrony, where agents are not simultaneously active at every time step. Following the model in~\citet{bonnefoi2017multi}, each agent \( m \in \mathcal{M} \) becomes active independently at time \(t\) with a fixed probability \( p \).The active agent set at time t is defined as  \( \mathcal{M}(t) \) 
We introduce a simplified model based on \textit{periodic activations}, which captures essential aspects of asynchrony while enabling structured coordination. Specifically, each agent \( m \in \mathcal{M} \) becomes active every \( \theta_m \in \mathbb{N}^+ \) rounds, i.e., at times \( \theta_m, 2\theta_m, \ldots, \left\lfloor T / \theta_m \right\rfloor \cdot \theta_m \). We define the activation frequency of agent \( m \) as \( \omega_m:= 1/\theta_m \). 

When active, an agent selects an arm and receives a reward defined as in the synchronous case. Collisions again yield zero reward. Let \( \mathcal{A}(t) \subseteq \mathcal{K} \) denote the set of arms selected at time \( t \), and \( \mathcal{I}(t) \) denote the top-\( |\mathcal{A}(t)| \) arms. At each round \( t \in [T] \), each active agent \( m \in \mathcal{M}(t) \) selects an arm \( k \in \mathcal{K} \),. Importantly, the size of the set \( \mathcal{M}(t) \) is equal to the size of the set \( \mathcal{I}(t) \), which represents the set of the top-\( |\mathcal{I}(t)| \) arms selected by the agents at time \( t \). Hence, both sets \( \mathcal{I}(t) \) and \( \mathcal{M}(t) \) are of equal size, ensuring that the number of arms selected corresponds to the number of active agents at each time step.
 The cumulative regret in the asynchronous setting is defined as
\[
\small
\mathbb{E}[R(T)] := \sum_{t=1}^{T} \sum_{k \in \mathcal{I}(t)} \mu(k) - \mathbb{E}_{\mu}\left[\sum_{t=1}^{T} \sum_{m \in \mathcal{M}(t)} r_t^{(m)}\right ] 
\]
where \( \mathcal{M}(t) \subseteq \mathcal{M} \) denotes the subset of agents active at time \( t \).
\subsection{Algorithm}
 Under the model defined in Section~\ref{sec: asych}, we propose an algorithm for periodic asynchronous activation that captures the essential challenges of asynchrony while introducing minimal structure.
 Each agent follows a fixed activation cycle in this setting, but agents are not synchronized, making traditional turn-based or globally scheduled strategies inapplicable.

To address this challenge, we model each time step within the global cycle as an independent instance of a Multi-Agent Multi-Armed Bandit (MMAB) problem. At each such step, only the currently active agents engage in interaction, following a staggered round-robin protocol to achieve collision-free and uniformly distributed exploration. This design facilitates globally balanced yet locally efficient exploration without requiring synchronization. To ensure no regret, each agent must accurately infer the ranking of the optimal arms based on the number of active agents present at its activation time.

A key challenge arises from heterogeneous activation frequencies: faster agents accumulate more information and may eliminate arms earlier than slower ones, leading to inconsistent arm sets. To maintain coordination, we introduce communication at synchronized time steps, i.e., the least common multiple (LCM) of agent cycles, where all agents exchange information via forced collisions. Additionally, fast agents broadcast updates on accepted/rejected arms to help slower agents align their decisions.

To reduce communication overhead, we adopt an arm-state tracking mechanism similar to Algorithm~\ref{alg:SynCD}, achieving low communication costs while ensuring effective information sharing. The full algorithm is provided in the Appendix.

\subsection{Theorems}
In this section, we establish the problem’s lower bound and analyze the theoretical guarantees of our algorithm.

\begin{theorem}
\label{thm:lowerbound}
Consider a periodic multiplayer bandit setting, where the number of active agents may vary over time. Let \( k_1^*, k_2^*, \ldots \) denote the smallest top arms at each corresponding time period, defining the dynamic optimal arm sets. Let \( \hat{\Delta}(k) \) represent the smallest reward gap between arm \( k \) and the corresponding critical arm \( k^* \) at each period. Then, for any uniformly efficient algorithm \( \pi \), the regret satisfies the following asymptotic lower bound:
\[
\small
\liminf_{T \to \infty} \frac{R_T(\pi,v)}{\log T} \ge \sum_{i: \hat{\Delta}(i) > 0} \frac{1}{\hat{\Delta}(i)},
\]
where the gap \( \hat{\Delta}(k) \) is defined as:
\[
\small
\hat{\Delta}(k) =
\begin{cases}
\infty, & k \leq k_1^* \\
\mu(k_1^*) - \mu(k), & k_1^* < k \leq k_2^* \\
\ldots \\
\mu(M) - \mu(k), & k > M
\end{cases}.
\]
\end{theorem}

\begin{theorem}
\label{periodic}
Consider a decentralized multi-agent bandit system where agents are activated periodically over a cycle of length \(l\), denoting the least common multiple of the agents' activation periods. For each time slot, let \(I\) denote the number of active agents, and let \(A(I)\) represent the empirical frequency that exactly \(I\) agents are active within one cycle. 
Define the suboptimality gap for each arm \(k > M\) as
\(\Delta(k) = \mu(M) - \mu(k)\)
where \(\mu(M)\) is the mean reward of the \(M\)-th best arm. 
For notational simplicity, we set \(\Delta(k) = \Delta_{\min}\) for all \(k \leq M\), where
\(\Delta_{\min} := \min_{k \leq M} \Delta(k)\),
while \(\Delta\) denotes the largest gap among them.

\textbf{(i) Group Regret.} \small
\[
\begin{aligned}
&\mathbb{E}[R(T)] = O\Bigg(
     \sum_I A(I) \Bigg[
        \frac{\Delta}{\Delta_{\min}} \sum_{I < k \leq M} \frac{\log T}{\Delta_{\min}} 
        + \sum_{k > M} \left(  \frac{(2+\beta)\Delta}{\Delta_{\min}} \right) \frac{\log T}{\Delta(k)} 
    \Bigg] \\
    &\quad+ \sum_{k > M} \frac{\log T}{\Delta(k)} + 2M^2 l \sum_{m=1}^{M} \left( \sum_{k=1}^K \log_{\beta}\left( \frac{8\beta}{\Delta(k)} \right) - 1 \right) C  + 2M^3 K l\\
    &\quad+2M^3 l \cdot \left( 1 + \frac{1}{2} \log \left( \frac{\log T}{\beta^2} + MK \right) \right)
\Bigg).
\end{aligned}
\]

\small \noindent where \(\sum_I A(I) = 1\).
\begin{figure*}[t]
\centering
\captionsetup[subfigure]{justification=centering, font=small, skip=2pt}

\begin{subfigure}[t]{0.24\textwidth}
    \centering
    \includegraphics[width=\linewidth]{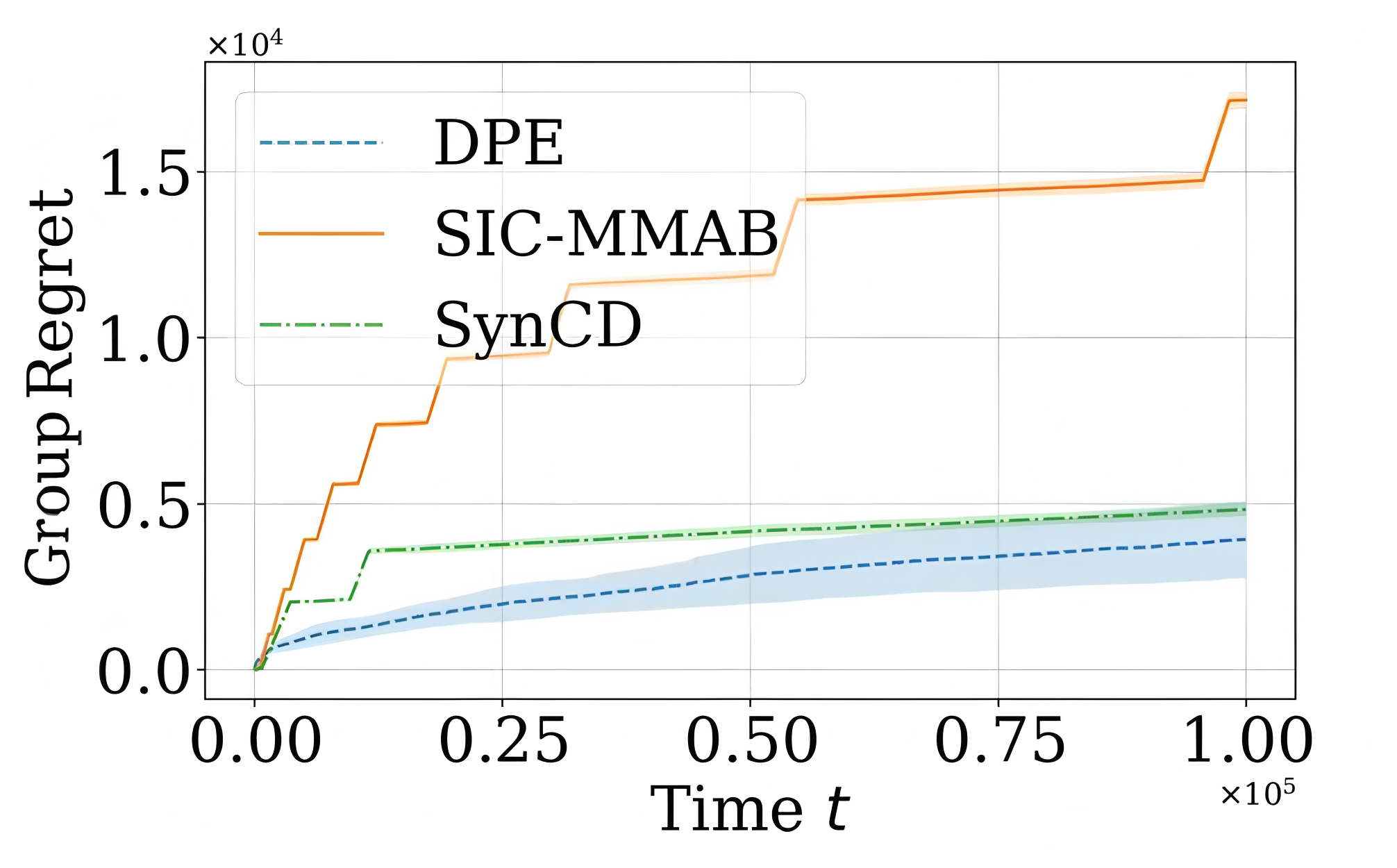}
    \caption{Group regret vs. baseline algorithms.}
    \label{fig:img1}
\end{subfigure}%
\hfill
\begin{subfigure}[t]{0.24\textwidth}
    \centering
    \includegraphics[width=\linewidth]{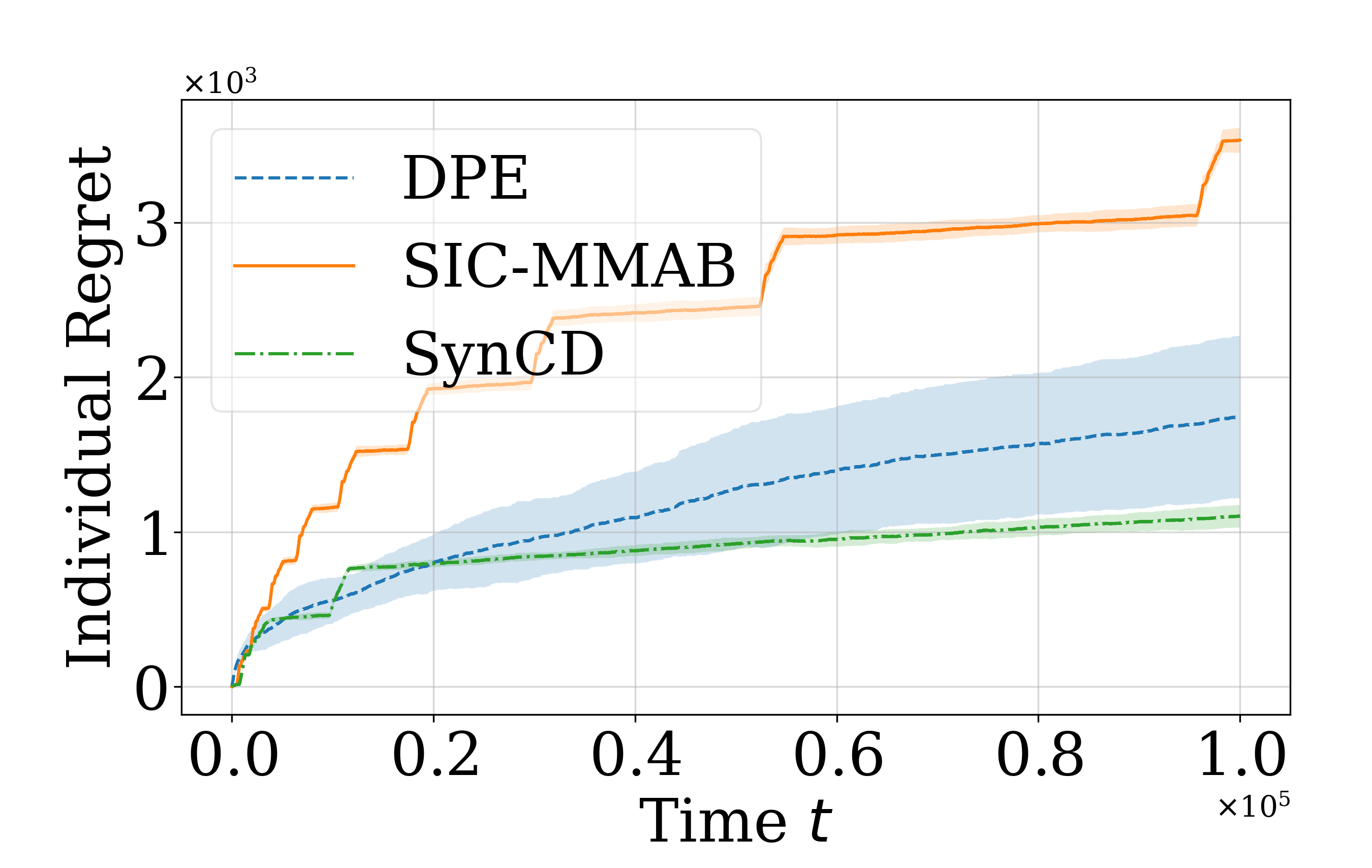}
    \caption{Individual regret vs. baseline algorithms.}
    \label{fig:img2}
\end{subfigure}%
\hfill
\begin{subfigure}[t]{0.24\textwidth}
    \centering
    \includegraphics[width=\linewidth]{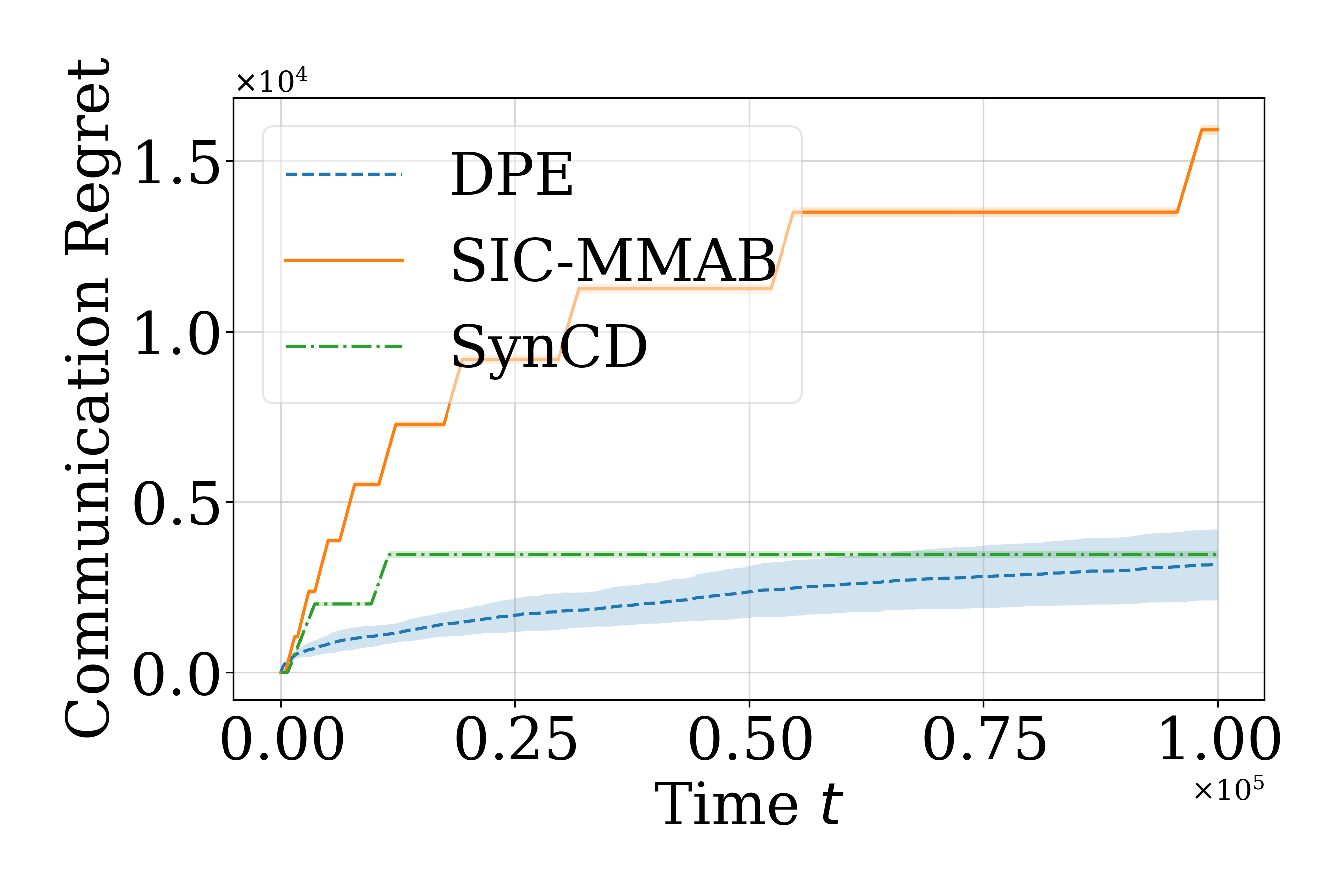}
    \caption{Communication cost (in regret).}
    \label{fig:img3}
\end{subfigure}%
\hfill
\begin{subfigure}[t]{0.24\textwidth}
    \centering
    \includegraphics[width=\linewidth]{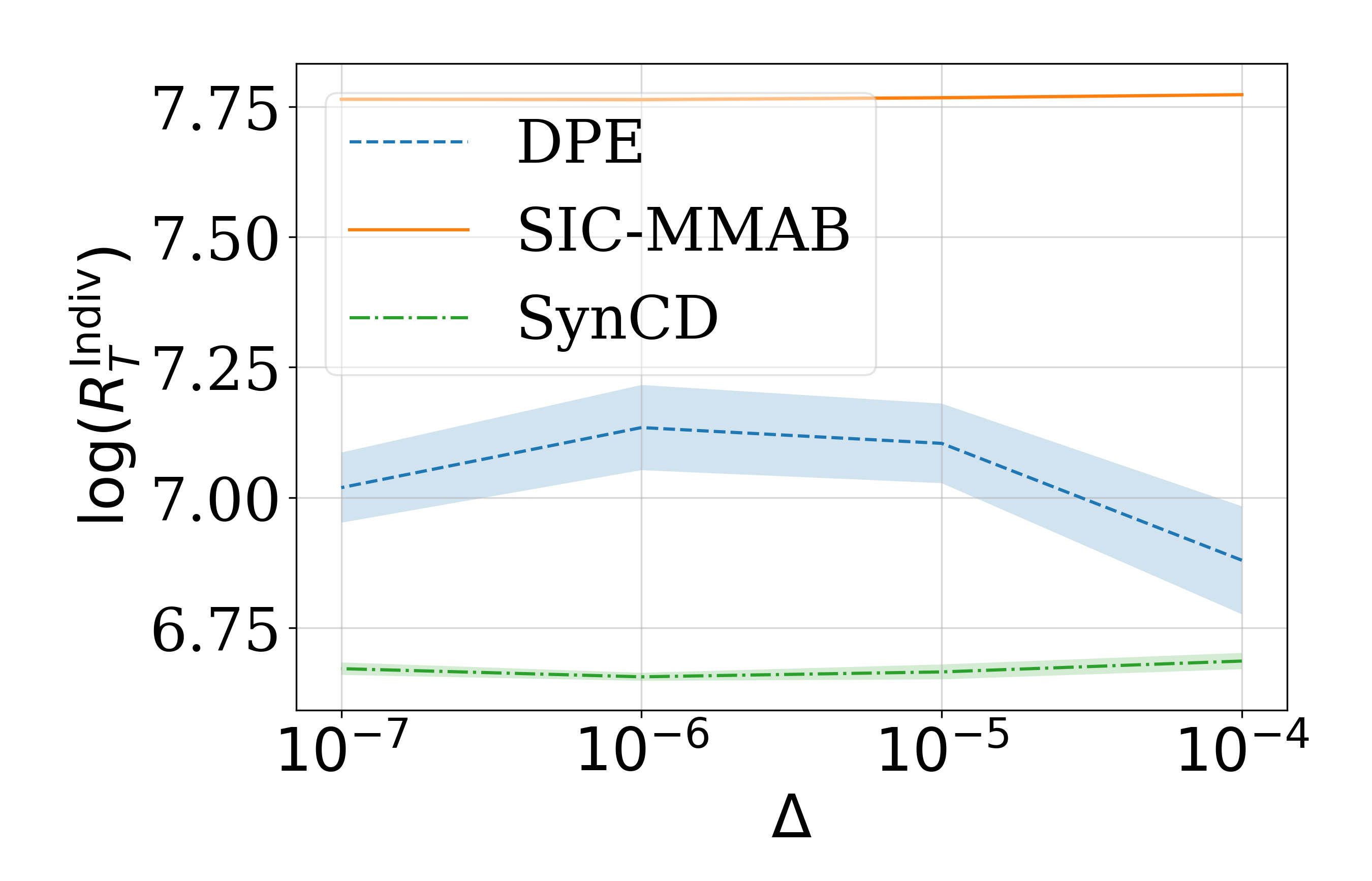}
    \caption{Individual regret under varying $\Delta$.}
    \label{fig:img4}
\end{subfigure}

\caption{
Comparison between our proposed algorithm and baselines listed in Table~\ref{tab:mmab-comparison}. Each subplot shows a different performance metric.
}
\label{fig:}
\end{figure*}

\end{theorem}
\noindent Let \(\alpha_m = \frac{\omega^{(m)}}{\sum_{i=1}^M \omega_i}\). Define the constant
\(C := 7 + \log_2\left(1 + \beta + \beta \sqrt{\frac{\alpha_m \ln 2}{2}}\right)\). 
The regret bound in Theorem~\ref{periodic} exhibits logarithmic dependence on time \(T\), consistent with classical lower bounds in stochastic bandits. The additional factor \(\Delta / \Delta_{\min}\) arises due to the need to sort the top \(M\) arms, where the minimal gap \(\Delta_{\min}\) determines the difficulty of ranking.

To better understand this effect, consider a simplified scenario where all suboptimal arms share the same gap, and \(\Delta_{\min} = \Delta\). Specifically, assume \(\Delta(k) = \Delta\) for all \(k > M\). Under this uniform-gap assumption, the main exploration regret simplifies to:
\[
R^{\text{explo}}[T] \leq c\sum_{I} A(I) (M-I) \frac{2\log T}{\Delta}
+ (K-M) \frac{(3+\beta)c\log T}{\Delta}
\]

In this case, the regret remains logarithmic in \(T\). The second term matches the order of the lower bound we provide, while the first captures the cost of cautious exploration in suboptimal configurations.

\section{Numerical Results}

This section presents numerical experiments that validate the performance of \texttt{SynCD}, highlighting its advantages in individual regrets and communication costs over state-of-the-art baselines.

\textbf{Setups and Baselines.} 
In the \texttt{SynCD} algorithm, we set parameters \(\beta = 4\) and \(\delta_t = 1/T^2\). 
We run \(20\) trials for each experiment under the setting \(K = 10\), \(M = 5\), and a time horizon of \(T = 5\times10^4\). 
Arm rewards are sampled from independent \textit{Bernoulli distributions} with means linearly spaced between 0.9 and 0.89, 
so the gap is approximately \(\Delta(k) = \frac{0.9 - 0.89}{10 - 1} \approx 1.11 \times 10^{-3}\). 
The results are plotted as the means of the trials (as lines) with their standard deviations shown as shaded regions. We note that $\beta$ is a tunable parameter controlling the communication frequency: a larger $\beta$ results in less frequent communication rounds, trading off communication cost and coordination accuracy.

We compare the regret and communication costs of \texttt{SynCD} with two baselines: \texttt{SIC-MMAB}~\citep{boursier2019sic} and \texttt{DPE1}~\citep{wang2020optimal}, as outlined in Table~\ref{tab:mmab-comparison}. All other baseline parameters follow their default choices.

\textbf{Observation 1.} Figures~\ref{fig:img1}--\ref{fig:img3} compare \texttt{SynCD} with baseline algorithms in terms of group regret, individual regret, and communication cost. Overall, \texttt{SynCD} achieves a good balance across all metrics.
In Figure~\ref{fig:img1}, \texttt{SynCD} shows lower group regret than \texttt{SIC\_MMAB} but slightly higher than \texttt{DPE1}, mainly due to its use of arm elimination rather than UCB-style exploration, which typically yields better group performance.
In Figure~\ref{fig:img2}, \texttt{SynCD} achieves the \textbf{lowest maximum individual regret}. Unlike \texttt{DPE1}’s leader-follower setup---which concentrates regret on the leader---\texttt{SynCD}’s distributed design ensures more balanced learning among agents.
Figure~\ref{fig:img3} shows that \texttt{DPE1} has the lowest communication cost, benefiting from its centralized structure where only the leader broadcasts updates. In comparison, \texttt{SynCD} still requires sharing arm-level information among agents, leading to a higher cost. However, \texttt{SynCD} remains significantly more efficient than \texttt{SIC\_MMAB}, while maintaining a fully decentralized design.

\textbf{Observation 2.}
To further evaluate the performance of \texttt{SynCD}, we present in Figure~\ref{fig:img4} the comparison of individual regrets under varying reward gaps~$\Delta$. The y-axis (in logarithmic scale) of the figure corresponds to the individual regrets measured at the end of the time horizon. \texttt{SynCD} achieves lower individual regrets across varying gaps, demonstrating greater robustness and stability with respect to~$\Delta$. Moreover, \texttt{SynCD} exhibits better performance when the reward gap is small, highlighting its effectiveness in more challenging scenarios.

\section{Conclusion}
This paper presents \texttt{SynCD}, a distributed algorithm for multi-agent bandit problems with collisions. \texttt{SynCD} achieves optimal group and individual regret with a low \(O(\log \log T)\) communication overhead. 
We then extend \texttt{SynCD} to address the more challenging asynchronous case, where agents have periodic activations.
To corroborate the effectiveness of the asynchronous \texttt{SynCD}, we also establish a regret lower bound. 
Experiments confirm that \texttt{SynCD} outperforms existing methods, especially in individual regret and communication cost.  For our asynchronous analysis, the proposed algorithm does not yet match the lower bound we proved.  Designing theoretically optimal algorithms under general periodic-asynchronous activation remains an important direction for future work.

\section{Supplementary materials for Algorithms}
\subsection{Detailed initialization phase for SynCD}
We define the initialization procedure as \textsc{InitPhase}.
The initialization procedure adopted in our algorithm follows \texttt{DPE1}~\citep{wang2020optimal}, which consists of two sub-phases: an \emph{orthogonalization phase} and a \emph{rank assignment phase}. This ensures that each agent is assigned a unique internal index $j \in \{0, \dots, M{-}1\}$ before exploration begins.

\paragraph{A. Orthogonalization.}
This sub-phase aims to assign $M$ different arms from the set $\{0, 1, \dots, K{-}1\}$ to the players in a fully decentralized manner. Each player maintains an internal \emph{state} in $\{0, 1, \dots, K{-}1\}$:
\begin{itemize}
    \item A player with state $0$ is unsatisfied and continues searching for a collision-free arm;
    \item A player with state $k \neq 0$ has found arm $k$ without collision and keeps this state for the remainder of the phase.
\end{itemize}

The orthogonalization phase proceeds in blocks of $(K+1)$ rounds:
\begin{itemize}
    \item In the first round of each block, satisfied players select their assigned arm (according to their state), while unsatisfied players randomly choose an arm from $\{0, 1, \dots, K{-}1\}$.
    \item The remaining $K$ rounds are used to detect whether all players are satisfied. During these rounds:
        \begin{itemize}
            \item Each satisfied player in state $k$ selects arm $k$ in all rounds except the $k$-th, where she switches to arm $K$;
            \item Unsatisfied players select arm $K$ throughout the $K$ rounds.
        \end{itemize}
\end{itemize}

If any player is still unsatisfied, collisions will occur at arm $K$ during these $K$ rounds. The absence of any collision signals that all players have found distinct arms, and the sub-phase can terminate. The expected duration of this sub-phase is bounded by $\frac{M(K-1)(K+1)}{K-M}$ rounds.

\paragraph{B. Rank Assignment.}
After orthogonalization, all players have distinct states in $\{1, \dots, K{-}1\}$, and the goal is to deterministically assign each one a unique internal rank.

This phase consists of $2K{-}2$ rounds labeled $t_1, \dots, t_{2K{-}2}$. A player in state $k$ selects arms in the following pattern:
\begin{itemize}
    \item In rounds $t_s$ where $s \in \{1, \dots, 2k\} \cup \{K{+}k, \dots, 2K{-}2\}$, she selects arm $k$;
    \item In rounds $t_s$ where $s \in \{2k{+}1, \dots, K{+}k{-}1\}$, she selects arm $s{-}k$.
\end{itemize}

Under this protocol, each pair of players collides exactly once. A player in state $k$ determines her rank by counting the number $i_k$ of collisions experienced during the first $2k$ rounds. And then we get the final rank.

\subsection{Detailed Communication Protocol for SynCD} 
\label{sec:append_comm}
We present the communication protocol in detail, described from the perspective of a single agent—a natural approach for decentralized algorithms. The protocol comprises three key subroutines: \textbf{Send Protocol}, \textbf{Receive Protocol}, and \textbf{Wait Protocol}, each governing the agent’s behavior based on its role in the communication round.

\begin{itemize}
    \item \textbf{Receive Protocol:} A receiver decodes information sent by another agent by pulling a designated arm and interpreting collision patterns as bits. This allows for decoding statistics of arbitrary precision.
    \item \textbf{Send Protocol:} A sender transmits a quantized statistic to a designated receiver by encoding it as a binary message and transmitting it through collision-based signaling.
    \item \textbf{Wait Protocol:} Agents not involved in the current communication wait without interfering. They pull a fixed arm until the communication concludes and then realign with the schedule.
\end{itemize}

The complete communication procedure is encapsulated in the \textsc{Comm} protocol, which systematically outlines the communication flow and operational details of each subroutine.

\begin{algorithm}[htbp]
\caption{\textsc{Comm}: Main Communication Protocol}
\textbf{Input:} $\delta_t^{(m)}(k)$ (local statistic), $j$ (internal rank) \\
\textbf{Output:} $\tilde{\mathbf{S}}$ (quantized statistics of all active players)
\begin{algorithmic}[1]
    \State $\tilde{\delta}_t^{(m)}(k) \gets \texttt{Quantization}(\delta_t^{(m)}(k))$
    \State $E_t := \{(i, l, k) \in [M] \times [M] \times \mathcal{K}_t \mid i \neq l\}$
    \For{$(i, l, k) \in E_t$}
        \If {$i = j$}
            \State \texttt{Send}($j, \tilde{\delta}_t^{(m)}(k), l$)
        \ElsIf{$l = j$}
            \State $\tilde{\mathbf{S}} \gets \texttt{Receive}(l)$
        \Else
            \State \texttt{Wait}($j, l$)
        \EndIf
    \EndFor
\end{algorithmic}
\end{algorithm}

\begin{algorithm}[htbp]
\caption{Receive Protocol}
\textbf{Input:} $l$ (sender’s internal rank)\\
\textbf{Output:} $\delta$ (decoded statistic)
\begin{algorithmic}[1]
    \State $\delta \gets 0$, $\pi \gets$ index of the $l$-th active arm
    \For{$n = 1, \dots, 2\left\lceil1+\log T(k)\right\rceil$}
        \State Pull $\pi$
        \If{$\eta_\pi(t) = 1$}
            \If{$n \bmod 2 = 0$}
                \State \textbf{break}
            \Else
                \State $\delta \gets \delta + 2^{n/2}$
            \EndIf
        \EndIf
    \EndFor
    \For{$n = 1, \dots, M_p$}
        \State Pull the $l$-th active arm
    \EndFor
    \State \textbf{return} $\delta \times 10^{-\left\lceil1+\log T(k)\right\rceil}$
\end{algorithmic}
\end{algorithm}

\begin{algorithm}[htbp]
\caption{Send Protocol}
\textbf{Input:} $l$ (receiver), $\tilde{\delta}$ (quantized statistic), $j$ (sender’s internal rank)
\begin{algorithmic}[1]
    \State $m \gets$ binary encoding of $\tilde{\delta}$ with sign; $L \gets \text{length}(m)$
    \For{$n = 1, \dots, 2L + |M_p|$}
        \If{$n \bmod 2 = 0$ and $n < 2L$}
            \If{$m_{n/2} = 1$}
                \State Pull the $l$-th active arm
            \Else
                \State Pull the $j$-th active arm
            \EndIf
        \ElsIf{$n \bmod 2 = 1$ and $n < 2L$}
            \State Pull the $j$-th active arm
        \ElsIf{$n = 2L$}
            \State Pull the $l$-th active arm
        \Else
            \State $j \gets (j + 1) \bmod |M_p|$
            \State Pull the $j$-th active arm
        \EndIf
    \EndFor
\end{algorithmic}
\end{algorithm}

\begin{algorithm}[htbp]
\caption{Wait Protocol}
\textbf{Input:} $j$ (own internal rank), $l$ (sender’s internal rank)
\begin{algorithmic}[1]
    \State $\pi \gets$ the $j$-th active arm
    \While{$\eta_\pi(t) \ne 1$}
        \State Pull $\pi$
    \EndWhile
    \For{$n = 1, \dots, |M_p| - j + l$}
        \State Pull $\pi$
    \EndFor
\end{algorithmic}
\end{algorithm}

\vspace{0.5em}
\noindent
Next, we introduce an auxiliary communication protocol, \textsc{Comm\_a}, designed for synchronizing binary arm states among players. It proceeds in two phases, each dedicated to either the \texttt{Acc} or \texttt{Rej} arm sets. In each phase, players signal the inclusion of arms via collision: a collision on a given arm implies it has been newly added to the corresponding set.

\begin{algorithm}[htbp]
\caption{\textsc{Comm\_a}: Arm State Synchronization via Collision}
\textbf{Input:} \texttt{Acc}, \texttt{Rej}, $j$ (own internal rank) \\
\textbf{Output:} Updated \texttt{Acc}, \texttt{Rej}
\begin{algorithmic}[1]
\State $E_t := \{(i, l) \in [M] \times [M] \mid i \neq l\}$
\For{$(i, l) \in E_t$}
    \For{$\text{mode} \in \{\texttt{Acc}, \texttt{Rej}\}$}
        \If{$i = j$} \Comment{Sender}
            \For{$k \in \mathcal{K}_t$}
                \State Pull arm $l$ if $k$ is newly marked in \texttt{mode}; otherwise, pull $j$
            \EndFor
        \ElsIf{$l = j$} \Comment{Receiver}
            \For{$k \in \mathcal{K}_t$}
                \State Pull arm $j$
                \If{collision detected}
                    \State Update \texttt{Acc} or \texttt{Rej} $\gets k$ based on current mode
                \EndIf
            \EndFor
        \Else \Comment{Waiter}
            \For{$k \in \mathcal{K}_t$}
                \State Pull arm $j$
            \EndFor
        \EndIf
    \EndFor
\EndFor
\end{algorithmic}
\end{algorithm}

\subsection{algorithms for asynchronous setting }
\label{sec:append_periodic}
In this section, we supplement the main content by presenting the algorithm and proof sketch for the periodic asynchronous setting, along with remarks on the lower bound and algorithm regret.
Let $n_t^{(m')}(k)$ be the total number of times agent $m'$ has pulled arm $k$ up to time $t$, and $n_{\text{last}}^{(m')}(k)$ be the number of pulls prior to the last synchronization. Similarly, let $X^{(m')}_t(k)$ and $X_{\text{last}}^{(m')}(k)$ denote the total accumulated rewards for arm $k$ at the current time and at the last synchronization, respectively.
Define $T_{\text{last}}^{(m')}(k)$ as the estimated number of total pulls (across agents) on arm $k$ prior to the last communication. Then, the updated estimator is given by:
\[
\hat{\mu}_{t}^{(m)}(k) = \frac{\displaystyle \sum_{m'=1}^{M} n_{\text{last}}^{(m')}(k)\,\tilde{\mu}^{(m')}(k) + X^{(m)}_t(k) - X_{\text{last}}^{(m)}(k)}{T_{\text{last}}(k) + n_t^{(m)}(k) - n_{\text{last}}^{(m)}(k)}
\]

Let $N_t(k)$ be the sample count for arm $k$, and define the confidence radius as $R(k,t) = 2\beta\sqrt{\frac{\log(1/\delta)}{2N_t(k)}}$, with confidence level $1 - \delta$

\begin{algorithm}[htbp]
\caption{Distributed Periodic-Asynchronous Algorithm for Agent $m$}
\begin{algorithmic}[1]
\State \textbf{Input:} total number of agents $M$, arm set $\mathcal{K}$, parameter $\beta > 1$
\State \textbf{Initialize:} For all $k \in \mathcal{K}$, set $n_t^{(m)}(k) \gets 0$, $X^{(m)}(k) \gets 0$, $\hat{\mu}^{(m)}(k) \gets 0$, $T_t(k) \gets 0$;
\State $\text{ECR} \gets 1$, $\text{ECR}_{\text{last}} \gets 1$, $\mathcal{K}_t \gets \mathcal{K}$, $K_t \gets K$, $M_t \gets M$, exploration schedule over $L = \text{lcm} \times K_t$ time steps
\While{Top-$M$ arms are not fully sorted}
    \State \textbf{Exploration Phase:}
    \For{$k = 1$ to $K_t$}
        \For{$\tau = 1$ to $L$}
            \State $t \gets t + 1$
            \If{$m \in \mathcal{A}(\tau)$} \Comment{ $m$ is active at time $\tau$}
                \State Let $j$ be local index of agent $m$ in $\mathcal{A}(\tau)$
                \State Let $|\mathcal{A}(\tau)| = M_\tau$
                \State Pull arm $a = (j + k) \bmod |\mathcal{K}_t|$
                \State Receive reward $r_t(a)$ 
                \State $X^{(m)}(a) \gets X^{(m)}(a) + r_t(a)$
                \State $n_t^{(m)}(a) \gets n_t^{(m)}(a) + 1$
            \EndIf
        \EndFor
    \EndFor
    \State $T_t(k) \gets T_t(k) + M_t$ \Comment{Update total pulls for arms}
    \State \textbf{Communication Phase:}
    \If{$\text{ECR}_t \leq \beta \cdot \text{ECR}_{\text{last}}$}
        \State $\text{ECR}_{\text{last}} \gets \text{ECR}_t$
        \State \texttt{Comm}$(\delta_t^{(m)}(k))$ \Comment{only at the end of each LCM}
        \State $T_{\text{last}}(k) \gets T_t(k)$, $n^{(m)}_{\text{last}}(k) \gets n_t^{(m)}(k)$
    \EndIf
    \State \textbf{Rejection Phase:}
    \[
    \text{Rej} 
     \gets \left\{
    k \in \mathcal{K}_t \ \middle|\ 
    \left|\left\{ i \in \mathcal{K}_t \mid \mathrm{LCB}^{(m)}_t(i) \geq \mathrm{UCB}^{(m)}_t(k) \right\}\right| \geq M_t
\right\}
    \]
    \State Remove rejected arms: $\mathcal{K}_t \gets \mathcal{K}_t \setminus \text{Rej}$
    \If{Rejection set changed and no communication occurred}
        \State \texttt{Comm\_index}$(\text{Rej})$ \Comment{Communicate rejection indices if needed and  only  at the end of each LCM period}
    \EndIf
    \State \textbf{Check Sorting Completion:}
    \State Let $\text{Top}_t = \{k_1, k_2, \dots, k_{M_t}\} \subset \mathcal{K}_t$ be the $M_t$ arms with highest $\hat{\mu}^{(m)}(k)$
    \If{For all $i < j$, $\mathrm{LCB}(k_i,t) > \mathrm{UCB}(k_j,t)$}
        \State Sorting completed; proceed to Exploitation Phase
    \Else
        \State Continue Exploration
    \EndIf
\EndWhile
\State \textbf{Exploitation Phase:}
\While{$t < T$}
    \State Let $\mathcal{A}(t)$ be the set of agents active at time $t$
    \State $m_t \gets |\mathcal{A}(t)|$
    \State Select the top $m_t$ arms: $\text{Top}_t \gets$ first $m_t$ arms in sorted $\text{Top}_t$
    \For{each agent $m \in \mathcal{A}(t)$}
        \State Let $j$ be the local index of agent $m$ in $\mathcal{A}(t)$
        \State Pull arm $\text{Top}_t[j+1]$
    \EndFor
\EndWhile
\end{algorithmic}
\end{algorithm}

This lower bound highlights the intrinsic difficulty introduced by periodic and heterogeneous agent activations. Unlike the classical MMAB lower bound, which only considers a fixed top-$M$ arm set, our formulation accounts for the fact that the set of critical arms varies across time, depending on the number of active agents. As a result, each suboptimal arm must be distinguished from its nearest critical arm, rather than a fixed arm like the M-th arm in traditional MMAB problems. This leads to a more refined and instance-dependent regret lower bound that captures the temporal complexity of the periodic setting.

\textbf{Proof Sketch of Theorem 3}:
We begin by constructing confidence intervals for the mean reward of each arm. The original problem is then reformulated into multiple Multi-Player Multi-Armed Bandit (MMAB) subproblems, each corresponding to a specific agent activation time. The total regret is decomposed into two components: (1) the regret due to not pulling the optimal arms in each MMAB instance, and (2) the regret from pulling suboptimal arms. The second component is further reorganized based on the corresponding MMAB instances, and each part is bounded using scaling and approximation techniques.
The proof for the communication regret follows a similar strategy as in communication regret in synchronous settings. We derive a constant upper bound on regret within a single phase. Due to the periodic nature of activations and we only communicate at the least common multiple of agents' periods. The total regret is then upper bounded by a constant factor times the least common multiple (LCM) of the agent periods.

\subsection{Proof for SynCD}
\begin{lemma}
Assume $M$ agents independently sample arms from an i.i.d. reward process with unknown mean $\mu(k)$. Let $n_t^{(m)}(k)$ be the number of samples that agent $m$ has collected for arm $k$ by time $t$, and define $T_t(k) = \sum_{m=1}^{M} n_t^{(m)}(k)$. Let $\mathrm{CR}_{[0,1]}(n, \delta)$ denote the confidence radius under Hoeffding’s inequality, and define $\mathrm{ECR}_t(k) = \mathrm{CR}_{[0,1]}(T_t(k), \delta)$, where $\delta \in (0,1)$ is a non-increasing sequence. Then, for any $t$, with probability at least $1 -\delta$, the following inequality holds:
\[
\left| \hat{\mu}_t^{(m)}(k) - \mu(k) \right| \leq 2\beta \cdot \mathrm{CR}_{[0,1]}(T_t(k), \delta)
\]
\end{lemma}
\begin{proof}
\begin{align}
&\left| \hat{\mu}_t^{(m)}(k) - \mu(k) \right| \notag \\
&= \left| \frac{ \sum_{m'=1}^{M} n_s^{(m')}(k) \tilde{\mu}^{(m')}(k) + X_t^{(m)}(k) - X_s^{(m)}(k) }{T_s(k) + n_t^{(m)}(k) - n_s^{(m)}(k)} - \mu(k) \right| \displaybreak[0] \notag \\
&= \Bigg| 
\frac{ \sum_{m'=1}^{M} n_s^{(m')}(k) \left( \tilde{\mu}^{(m')}(k) - \bar{\mu}^{(m')}(k) \right) }{T_s(k) + n_t^{(m)}(k) - n_s^{(m)}(k)} \notag \\
&\quad + \frac{ \sum_{m'=1}^{M} n_s^{(m')}(k) \bar{\mu}^{(m')}(k) + X_t^{(m)}(k) - X_s^{(m)}(k) }{T_s(k) + n_t^{(m)}(k) - n_s^{(m)}(k)} - \mu(k) 
\Bigg| \displaybreak[0] \notag \\
&\leq \left| \frac{ \sum_{m'=1}^{M} X_s^{(m')}(k) + X_t^{(m)}(k) - X_s^{(m)}(k) }{T_s(k) + n_t^{(m)}(k) - n_s^{(m)}(k)} - \mu(k) \right| \notag \\
&\quad + \left| \frac{ \sum_{m'=1}^{M} n_s^{(m')}(k) (\tilde{\mu}^{(m')}(k) - \bar{\mu}^{(m')}(k)) }{T_s(k) + n_t^{(m)}(k) - n_s^{(m)}(k)} \right| \displaybreak[0] \notag \\
&\leq \mathrm{CR}_{[0,1]}(T_s(k) + n_t^{(m)}(k) - n_s^{(m)}(k), \delta) + \sqrt{ \frac{1}{T_t(k)} } \displaybreak[0] \notag \\
&\stackrel{(a)}{\leq} \mathrm{CR}_{[0,1]}(T_s(k), \delta) + \sqrt{ \frac{1}{T_t(k)} } \displaybreak[0] \notag \\
&\stackrel{(b)}{\leq} 2\beta \cdot \mathrm{CR}_{[0,1]}(T_t(k), \delta) \notag
\end{align}

where (a) uses the fact that $\mathrm{CR}$ increases as the number of samples decreases, and (b) follows from the assumption that the communication condition is not met at time $t$.
\end{proof}

We decompose the regret as follows:
\begin{align}
\mathbb{E}[R(T)] = 
R_{\mathrm{init}} + R_{\mathrm{comm}} + R_{\mathrm{explo}}.
\end{align}
\begin{align}
\text{where} \quad
\left\{
\begin{aligned}
R^{\mathrm{init}} &= T_{\mathrm{init}} \sum_{k=1}^{M} \mu_{(k)} 
- \mathbb{E}_{\mu} \left[\sum_{t=1}^{T_{\mathrm{init}}} \sum_{j=1}^{M} r^{j}_t\right], \\
&\text{with } T_{\mathrm{init}} = T_0 + 2K, \\
R^{\mathrm{comm}} &= \mathbb{E}_{\mu} \left[ \sum_{t \in \mathrm{Comm}} \sum_{j=1}^{M} 
\left( \mu_{(j)} - r^{j}_t \right) \right], \\
&\text{with Comm the set of communication steps}, \\
R^{\mathrm{explo}} &= \mathbb{E}_{\mu} \left[ \sum_{t \in \mathrm{Explo}} \sum_{j=1}^{M} 
\left( \mu_{(j)} - r^{j}_t \right) \right], \\
&\text{with } \mathrm{Explo} = \{ T_{\mathrm{init}} + 1, \ldots, T \} \setminus \mathrm{Comm}.
\end{aligned}
\right.
\end{align}

Following the same initialization procedure as in \cite{wang2020optimal}, the resulting initialization regret can be directly bounded as a constant. Then we perform regret analysis by separately bounding the two components of the pseudo-regret: the \emph{communication regret} \( R^{\mathrm{comm}} \) and the \emph{exploration regret} \( R^{\mathrm{explo}} \), as defined previously.

\begin{lemma}
 With probability at least $1-\delta-M\exp(-\frac{T_0}{K})$, every optimal arm k is accepted
after at most $\left(\frac{32\beta^2\log\delta^{-1}}{\mu(k)-\mu(M)}+MK\right )$pulls during exploration phases, and every sub-optimal arm k is
rejected after at most $\left (\frac{32\beta^2\log\delta^{-1}}{\mu(M)-\mu(k)}+MK\right )$pulls during exploration phases\\
\end{lemma}
\begin{proof} 
If initialization is successful and the estimation of the mean of reward lies in the 
confidence interval, for any active optimal arm k, we have the following:
\begin{align*}
    2(2\beta \mathrm{CR}_{[0,1]}(T_{t}(k),\delta)+ 2(2\beta \mathrm{CR}_{[0,1]}(T_{t}(k),\delta) \ge \Delta(k)
\end{align*}
$\Delta(k) = \mu(k)-\mu(M+1)$ is the gap between arm k and first suboptimal arm. Otherwise, the arm will be accepted according to the elimination rules.
Then we have:
    \[4(2\beta\mathrm{CR}_{[0,1]}(T_v(k),\delta) \ge \Delta(k)\]
Since in the algorithm, each active agent pulls arm k once in each round without collision, we can prove that the number of pullings before arm i is accepted is upper  bounded by :  
\begin{align*}
     T_v(k) &\le \frac{32\beta^2\log \delta^{-1}}{\Delta(k)^2}+M,\quad \Delta(k)=(\mu(k)-\mu(M+1))
\end{align*}
so the optimal active arm k will be accepted after at most $(\frac{32\beta^2\log \delta^{-1}}{\Delta(k)^2} +  MK)$ pulls, $\Delta(k)=\mu(k)-\mu(M+1)$ and the analysis about when suboptimal arms will be rejected is similar to the one above and we will have $T_v(k)=\left ( \frac {32\beta^2\log \delta^{-1}}{\Delta(k)^2}+MK\right )$,$\Delta(k)=\mu(M)-\mu(k)$.\\
\end{proof}

\begin{lemma}
With probability at least $1- \delta-M\exp(-\frac{T_0}{K})$, the exploration regret is upper bounded by:
\[
R^{\text{explo}}
    \le (\sum_{k>M}\frac{ 32\beta^2(\beta+2)\log\delta^{-1}}{\mu(M)-\mu(k)})+MK(\mu(k)-\mu(M))
\]
\end{lemma}

\begin{proof}
We first define the $T^{\text{explo}}$ as the total time steps of exploration and exploitation. We denote $T^{\text{explo}}$ as the centralized number of exploration and exploitation of the k-th best arm.
There is no collision during both the exploration and exploitation phases, so the regret can be decomposed as
\begin{align*}
    R^{\text{explo}}=& \displaystyle \sum_{k>M}(\mu(M)-\mu(k))T^{\text{explo}}(k)\\
    & +\displaystyle \sum_{k\le M}(\mu(k)-\mu(M))(T^{\text{explo}}-T^{\text{explo}}(k))\\
\end{align*}
First, we consider $T^{\text{explo}}-T^{\text{explo}}_{(k)} (k\le M)$.
We denote $P$ as the total number of communication rounds. We denote $N_p$ as the total number of exploration phases between the $p$-th and the $p-1$-th communication round. 
We denote $\hat{t}_k$ as the number of exploratory pulls before accepting the arm $k$, and we denote $T_p$ as the number of exploratory pulls at the p-th communication roun,d and we have:\\
\[ T^{\text{explo}}\le T^{\text{explo}}(k) +\sum_{p=1}^{P}N_pM(K_p-M_p)\mathbf{1}_{\hat{t}_k>T_{p-1}}\]
Before the active optimal arm k is accepted by the agent $M$, in each phase it will be pulled $M\cdot M_p$ times. If arm k has been accepted, it will be pulled for $M\cdot K_p$ times in each round, so in every round before $t_p$, the arm is not pulled $M(K_p-M_p)$ times, so we have the inequality above.\\
With the equality below:
\[K_p-M_p=\sum_{j>M}\mathbf{1}_{\hat{t}_j>T_{p-1}}\]
We have:
\begin{align*}
    &\sum_{k\le M}(T^{\text{explo}} - T^{\text{explo}}(k))(\mu(k)-\mu(M))\\
    &\le \sum_{k\le M}(\mu(k)-\mu(M))\sum_{p=1}^{P}N_pM\sum_{j>M}\mathbf{1}_{\hat{t}_j>T_{p-1}}\mathbf{1}_{\hat{t}_k>T_{p-1}}\\
    &\le \sum_{j>M}\sum_{p=1}^{P}\sum_{k\le M} N_pM(\mu(k)-\mu(M))\mathbf{1}_{\min(\hat{t}_j,\hat{t}_k)>T_{p-1}}\\
\end{align*}
Let's define $A_j$ as
\begin{align*}
    A_j &=\sum_{p=1}^{P}\sum_{k\le M}N_pM(\mu(k)-\mu(M))\mathbf{1}_{\min(\hat{t}_j,\hat{t}_k)>T_{p-1}}\\
    &=\sum_{p=1}^{P_j}\sum_{k\le M}N_pM(\mu(k)-\mu(M)\mathbf{1}_{\hat{t}_k>T_{p-1}}
\end{align*}
According to Lemma 4.3, we have:
\begin{align*}
    T_{p-1} \le \frac{32\beta^2 \log\delta^{-1}}{\Delta(k)^2} +  MK,\quad p<N
\end{align*}
then we yield 
\begin{align*}
    \Delta(k) \le \sqrt{\frac{32\beta^2 \log\delta^{-1}}{T_{p-1}-MK}}
\end{align*}
Then it follows:
\begin{align*}
A_j &= \sum_{p=1}^{P_j} \sum_{k \le M} N_p M \left( \mu(k) - \mu(M) \right) 
\mathbf{1}_{\hat{t}_k > T_{p-1}} \\
&\le \sum_{p=1}^{P_j} \sum_{k \le M} 
N_p M \sqrt{ \frac{32\beta^2 \log \delta^{-1}}{T_{p-1} - MK} }
\mathbf{1}_{\hat{t}_k > T_{p-1}} \\
&= \sqrt{32\beta^2 \log \delta^{-1}} 
\sum_{p=1}^{P_j} \sqrt{ \frac{1}{T_{p-1} - MK} } \cdot N_p M \cdot M_p \\[0.5ex]
&= \sqrt{32\beta^2 \log \delta^{-1}} 
\sum_{p=1}^{P_j} \sqrt{ \frac{1}{T_{p-1} - MK} } \cdot (T_p - T_{p-1}) \\[0.5ex]
&= \sqrt{32\beta^2 \log \delta^{-1}} 
\sum_{p=1}^{P_j} (\sqrt{T_p - MK} - \sqrt{T_{p-1} - MK}) \\
&\hspace{4em} \cdot (\sqrt{T_p - MK} + \sqrt{T_{p-1} - MK}) 
\cdot \sqrt{ \frac{1}{T_{p-1} - MK} } \\
&= \sqrt{32\beta^2 \log \delta^{-1}} 
\sum_{p=1}^{P_j} (\sqrt{T_p - MK} - \sqrt{T_{p-1} - MK}) \\
&\hspace{4em} \cdot \left( \sqrt{ \frac{T_p - MK}{T_{p-1} - MK} } + 1 \right) \\
&\le (\beta + 1) \sqrt{32\beta^2 \log \delta^{-1}} 
\sum_{p=1}^{P_j} (\sqrt{T_p - MK} - \sqrt{T_{p-1} - MK}) \\
&\le (\beta + 1) \sqrt{32\beta^2 \log \delta^{-1}} 
\cdot \sqrt{T_{P_j} - MK} \\
&\le (\beta + 1) \sqrt{32\beta^2 \log \delta^{-1} \cdot (\hat{t}_j - MK)}
\end{align*}

as $ \hat{t}_j \le \frac{32\beta^2\log\delta^{-1}}{(\mu(M)-\mu(j))^2} + MK$,so we have 
\begin{align*}
    A_j \le (\beta+1)\frac{32\beta^2 \log\delta^{-1}}{\mu(M)-\mu(j)} 
\end{align*}
Then we yield:
\begin{align*}
     &\sum_{k\le M}(T^{\text{explo}} - T^{\text{explo}}(k))(\mu(k)-\mu(M)\\
     & \le \sum_{j> M}(\beta+1)\frac{32\beta^2 \log\delta^{-1}}{\mu(M)-\mu(j)} 
\end{align*}

Then, we can proof for any suboptimal arm k that:
\begin{align*}
    &(\mu(M) - \mu(k))T^{\text{explo}}(k) \\&\le (\mu(M) - \mu(k))  (\frac{32\beta^2\log\delta^{-1}}{(\mu(M)-\mu(k))^2}+MK)\\
    &\le \frac{32\beta^2\log\delta^{-1}}{(\mu(M)-\mu(k))}+ MK(\mu(M)-\mu(k))
\end{align*}
Above all, we can get the final exploration regret upper bound
\begin{align*}
R^{\text{explo}} 
&= \sum_{k > M} (\mu(M) - \mu(k)) \cdot T^{\text{explo}}(k) 
\\
& + \sum_{k \le M} (\mu(k) - \mu(M)) 
\cdot \left( T^{\text{explo}} - T^{\text{explo}}(k) \right) \\[0.5ex]
&\le \sum_{k > M} 
\frac{32\beta^2 (\beta+2) \log \delta^{-1}}{\mu(M) - \mu(k)} 
+ MK (\mu(M) - \mu(k))s
\end{align*}

\end{proof}

\begin{lemma}
\label{lemma:comm_round}
The total number of communication rounds can be upper bounded by:
\begin{align*}
    \sum_{k\le M}\log_{\beta}(\frac{8\beta)}{(\mu(k)-\mu(M+1))})+\sum_{k>M}\log_{\beta}(\frac{8\beta)}{(\mu(M)-\mu(k))})
\end{align*}
\end{lemma}
\begin{proof}
Let $\tau _k$  be the \text{last} round to pull the Top-M optimal arm k, after which k will be accepted. According to rules that accept arms, when the mean gap between arm k and the first suboptimal arm $M+1$ $\Delta(k)=\mu(k)-\mu(M+1)$, which is at least active until 
 The next round holds that 
\begin{align*}
    4(2\beta\text{CR}_{[0.1]}(T_{\tau_k},\delta)) \ge \Delta(k)
\end{align*}
And we can get:
\begin{align*}
    8\beta\text{ECR}_{\tau_k}(k)\ge \Delta(k)
\end{align*}
Because of the check condition, we can evaluate the times to communicate as follows:
\begin{align*}
    \log_{\beta}(\frac{ECR_1(k)}{ECR_{\tau_k}(k)}) \le  \log_{\beta}(\frac{8\beta}{ \Delta(k)})\le \log_{\beta}(\frac{8\beta}{(\mu(k)-\mu(M+1))})
\end{align*}
So the expected number of communications by optimal arms is at most 
\begin{align*}
    \sum_{k\le M}\log_{\beta}(\frac{8\beta}{(\mu(k)-\mu(M+1))})
\end{align*}
the analysis is same on suboptimal arms $\displaystyle \sum_{k>M}\log_{\beta}(\frac{8\beta}{(\mu(M)-\mu(k))})$ and sum up, the total number of communication rounds overheads is upper bounded by
\begin{align*}
    \sum_{k\le M}\log_{\beta}(\frac{8\beta}{(\mu(k)-\mu(M+1))})+\sum_{k>M}\log_{\beta}(\frac{8\beta}{(\mu(M)-\mu(k))})
\end{align*}
\end{proof}
\begin{lemma}
\label{lemma:quant_bits}
For arm statics, at time t, when communication is triggered, we quantize the means as
$\tilde{\mu}_{t}^{(m)}(k) = \textbf{ceil}(\bar{\mu}_t^{(m)}(k))$, with $\left\lceil 1 + \frac{\log T_t(k)}{2} \right\rceil $ bits, where $T_t(k)$ is the total number of times explored by all agents and we further communicate the truncated version of the difference between estimated mean $\tilde{\mu}_{s}^{(m)}(k)$ from last communication time s. By applying the adapting quantization strategy, in expectation, the ($7+ \log_2(1+\beta+\sqrt{\beta^2M\ln 2/2}$) bits are sufficient to represent the truncated version of 
 $\tilde{\delta}^{(m)}_{t}(k)$
\end{lemma}
\begin{proof}
the quantization leads to a quantization error of at most $\sqrt{\frac{1}{T_{t}(k)}}$ and have the following holds:
\begin{align*}
    \left| \tilde{\mu}_{t}^{(m)}(k)- \bar{\mu}_{t}^{(m)}(k)  \right| \le \sqrt{\frac{1}{T_{t}(k)}}
\end{align*}
We set the truncated version of the difference $\tilde{\delta}^{(m)}_{t}(k)$ and it can be bounded as  
\begin{align*}
\left| \tilde{\delta}^{(m)}_{t}(k) \right| 
&= \left| \tilde{\mu}_{t}^{(m)}(k) - \tilde{\mu}_{s}^{(m)}(k) \right| \\[0.5ex]
&= \Big| \tilde{\mu}_{t}^{(m)}(k) - \bar{\mu}_{t}^{(m)}(k)
- \big( \tilde{\mu}_{s}^{(m)}(k) - \bar{\mu}_{s}^{(m)}(k) \big) \\
&\hspace{3em} + \bar{\mu}_{t}^{(m)}(k) - \bar{\mu}_{s}^{(m)}(k) \Big| \\[0.5ex]
&\le \left| \tilde{\mu}_{t}^{(m)}(k) - \bar{\mu}_{t}^{(m)}(k) \right| 
+ \left| \tilde{\mu}_{s}^{(m)}(k) - \bar{\mu}_{s}^{(m)}(k) \right| \\
&\quad + \left| \bar{\mu}_{t}^{(m)}(k) - \bar{\mu}_{s}^{(m)}(k) \right| \\[0.5ex]
&\le \sqrt{ \frac{1}{T_{t}(k)} } 
+ \sqrt{ \frac{1}{T_{s}(k)} } 
+ \left| \bar{\mu}_{t}^{(m)}(k) - \bar{\mu}_{s}^{(m)}(k) \right|
\end{align*}

Further, we consider $\bar{\mu}_{t}^{(m)}(k)-\bar{\mu}_{s}^{(m)}(k)$,we can rewrite it as 
\begin{align*}
    &\bar{\mu}_{t}^{(m)}(k)-\bar{\mu}_{s}^{(m)}(k)
    =\frac{\sum^{t}_{\text{p=1}}\gamma^{(m)}_{\text{p}}(k)}{n_{t}(k)}- \frac{ \sum^{s}_{\text{p=1}}\gamma^{(m)}_{\text{p}}(k)}{n_{s}(k)}\\
    &=\frac{ \sum^{s}_{\text{p=1}}\gamma^{(m)}_{\text{p}}(k)+\sum^{t}_{\text{p=s}}\gamma_{s}^{(m)}(k)}{n_{t}(k)}- \frac{ \sum^s_{\text{p=1}}\gamma^{(m)}_{\text{p}}(k)}{n_{s}(k)}\\
\end{align*}
which is a $\frac{1}{2}\sqrt{\frac{1}{n_{s}(k)}-\frac{1}{n_{t}(k)}}$ sub-Gussian random variable since the unility samples are independent across time. Thus, we can further derive that, with a dummy variable x,
\begin{align*}
    &\mathbb{P}\left(
        \left| \bar{\mu}_{t}^{(m)}(k) - \bar{\mu}_{s}^{(m)}(k) \right|
        \ge \sqrt{ \frac{x^2}{T_t(k)} }
    \right) \\
    &\le 2 \exp\left[ - \frac{2x^2}{T_t(k)} \cdot 
        \frac{1}{ \frac{1}{n_s(k)} - \frac{1}{n_t(k)} }
    \right] \\[1ex]
    \Rightarrow\quad 
    &\mathbb{P}\left(
        \left| \bar{\mu}_{t}^{(m)}(k) - \bar{\mu}_{s}^{(m)}(k) \right|
        \ge \sqrt{ \frac{x^2}{T_t(k)} }
    \right) 
    \le 2 \exp\left[ - \frac{2x^2}{\beta^2 M} \right] \\[1ex]
    \Rightarrow\quad 
    &\mathbb{P}\left(
        \left| \tilde{\delta}^{(m)}_{t}(k) \right| 
        \ge \sqrt{ \frac{1}{T_t(k)} } + \sqrt{ \frac{1}{T_s(k)} } 
        + \sqrt{ \frac{x^2}{T_t(k)} }
    \right) \\
    &\le 2 \exp\left[ - \frac{2x^2}{\beta^2 M} \right] \\[1ex]
    \overset{(a)}{\Rightarrow}\quad 
    &\mathbb{P}\left(
        L_t^{(m)}(k) 
        > 3 + \frac{\log T_t(k)}{2} 
        + \log_2\left( \frac{1 + \beta + x}{ \sqrt{T_t(k)} } \right)
    \right)\\ 
    &\le 2 \exp\left[ - \frac{2x^2}{\beta^2 M} \right] \\[1ex]
    \Rightarrow\quad 
    &\mathbb{P}\left(
        L_t^{(m)}(k) \le 3 + \log_2(1 + \beta + x)
    \right) 
    \ge 1 - 2 \exp\left[ - \frac{2x^2}{\beta^2 M} \right] \\[1ex]
    \overset{(b)}{\Rightarrow}\quad 
    &\mathbb{P}\left(
        L_t^{(m)}(k) \le l
    \right) 
    \ge 1 - 2 \exp\left[ 
        - \frac{2(2^{l - 3} - 3)^2}{\beta^2 M} 
    \right]
\end{align*}

where $L^r_{k,m}$ in implication (a) is the length of the truncated version $|\tilde{\delta}^r_{k,m}|$ and is upper bounded by 
\begin{align*}
     L_{t}^{(m)}(k)&\leq \lceil 1+\log T_{t}(k)/2\rceil-\lfloor\log_2(1/|\tilde{\delta}_{t}^{(m)}(k)|)\rfloor\\
     &\leq 3 + \frac{\log{T_{t}}(k)}{2}+\log_2(|\tilde{\delta}_{t}^{(m)}(k)|).
 \end{align*}  
 In deriving (b), we substitute the variable $3+\log_2(1+\beta+x)$ with $l$, which satisfies $l\geq 1+\beta+\log_2(3+\sqrt{\beta^2M\ln 2/2})$, and thus equivalently $x = 2^{l-3}-1-\beta$. With the above results and viewing $L_{t}^{(m)}(k)$ as a random variable, we have that its cumulative distribution function (CDF) $F_{L_{t}^{(m)}(k)}(l)$ satisfies the following property: $\forall l > b$, where $b = \lceil3+\log_2(1+\beta+\sqrt{\beta^2M\ln 2/2})\rceil$
\begin{align*}
F_{L_{t}^{(m)}(k)}(l) = \mathbb{P}\left(L_{r}^{(m)}(k) \leq l\right)\geq 1-2\exp\left[-\frac{(2^{l-3}-1-\beta)^2}{\beta^2M}\right].
\end{align*}
Using the property of CDF, we can bound the expectation of $L_{t}^{(m)}(k)$ as
\begin{align*}
    \mathbb{E}\left[L_{t}^{(m)}(k)\right] &= \sum_{l = 0}^{\infty} (1-F_{L^r_{k,m}}(l))\\
    &\leq b+ \sum_{l=b}^{\infty}2\exp\left[-\frac{(2^{l-3}-1-\beta)^2}{\beta^2M}\right]\\
    &\leq b+ \int_{l=b}^{\infty}2\exp\left[-\frac{(2^{l-3}-1-\beta)^2}{\beta^2M}\right]\mathrm{d} l\\
    &\leq b+2\\
    &\leq 6+ \log_2(1+\beta+\sqrt{\beta^2M\ln 2/2})
\end{align*}
Thus, we have that in expectation, the truncated version of $\tilde{\delta}^{(m)}_{t}(k)$ has a length that is less than $6+ \log_2(1+\beta+\sqrt{\beta^2M\ln 2/2}$ bits. In addition, 1 bit of information should also be transmitted to indicate the sign of $\tilde{\delta}^{(m)}_{t}(k)$. As a
summary, in expectation, $7+ \log_2(1+\beta+\sqrt{\beta^2M\ln 2/2}$ bits are sufficient to represent the truncated version of $\tilde{\delta}^{(m)}_{t}(k)$.
\end{proof}

\begin{lemma}
The number of bits required to transmit the first estimated mean between agents in our differential quantization scheme is upper bounded by
\[
    \left\lceil 1 + \frac{1}{2} \log \left( \frac{\log(\delta^{-1})}{2 \beta^{2}} + M K \right) \right\rceil,
\]
which implies an overall communication cost of order \( \mathcal{O}(\log \log T) \).
\end{lemma}

\begin{proof}
In our threshold-based triggering mechanism, we initialize the Estimated Confidence Radius (ECR) to 1. A communication is triggered whenever the current ECR satisfies 
\[
ECR_t \leq \beta \times ECR_{\mathrm{last}},
\]
where \( \beta \in (0,1) \) is a fixed threshold parameter. According to Lemma~\ref{lemma:comm_round}, this guarantees a constant number of communications.

At the time of the first communication, by the definition of ECR, we have
\[
\sqrt{\frac{\log(\delta^{-1})}{2 T_t(k)}} \leq \beta \times 1,
\]
which can be rearranged to give
\[
T_t(k) \leq \frac{\log(\delta^{-1})}{2 \beta^{2}} + M K.
\]

In our differential quantization strategy, the mean reward is quantized using 
\[
\left\lceil 1 + \frac{\log T_t(k)}{2} \right\rceil
\]
bits. Since this is the first communication, there is no previous mean estimate available for differential encoding, so the full quantized mean must be transmitted.

Substituting the bound on \( T_t(k) \), we have that the number of bits for the first communication is bounded by
\[
\left\lceil 1 + \frac{1}{2} \log \left( \frac{\log(\delta^{-1})}{2 \beta^{2}} + M K \right) \right\rceil,
\]
which grows asymptotically as \( \mathcal{O}(\log \log T) \) when \( T \to \infty \).
\end{proof}

Finally, besides the communication cost from the \texttt{Comm} phase, which can be directly obtained by multiplying the bounds given in Lemma~\ref{lemma:comm_round} and Lemma~\ref{lemma:quant_bits}, yielding a constant communication overhead and adding the main cost incurred by the transmission of the first estimated mean, we also need to upper bound the cost incurred during the \texttt{Comm\_A} phase.

Specifically, synchronization may be triggered at most \( K \) times, with each synchronization communication costing at most \( 2M^3 \) bits, resulting in an additional total cost bounded by \( 2M^3 K \).

Putting all components together, the total expected communication cost is upper bounded by
\[
\begin{aligned}
\mathbb{E}[R_C] = O\Bigg(&
2M^3 \left(\sum_{k=1}^K \log_{\beta}\left(\frac{8 \beta}{\Delta(k)}\right)-1\right)
\cdot C \\
& + 2M^3 K \\
& + 2M^3 \cdot \left(
1 + \frac{1}{2} \log \left( \frac{\log(\delta^{-1})}{2 \beta^{2}} + M K \right)
\right)
\Bigg),
\end{aligned}
\]

where

\[
C := 7 + \log_2\left(1 + \beta + \sqrt{\frac{M \ln 2}{2}}\right).
\]

\subsection{Proof for extended algorithm}
\begin{lemma}
For the periodic setting, the problem reduces to a sequence of multiplayer bandit instances, where the number of active agents at each time defines the top arm set size. Let \(k_1^*, k_2^*, \ldots\) denote the smallest top arms at each time point. The goal is to distinguish each arm \(k\) from its corresponding critical arm \(k^*\).
Then the regret lower bound satisfies,
\begin{align*}
\liminf_{T \to \infty} \frac{R_T(\pi,v)}{\log T} \ge \sum_{i: \hat{\Delta}(i) > 0} \frac{1}{\hat{\Delta}(i)},\\
\text{ where }
\hat{\Delta}(k) = \begin{cases}
\infty, & k \leq k_1^* \\
\mu_{k_1^*} - \mu_k, & k_1^* < k \leq k_2^* \\
\vdots & \\
\mu_{M} - \mu_k, & k > M
\end{cases}.
\end{align*}
\end{lemma}

\begin{proof}
We begin by introducing several notations that will be used throughout the proof.

Let $\mathcal{I}$ denote the set of optimal arms. For any given time step with $|I|$ active agents, we denote by $I \subseteq \mathcal{I}$ the subset of optimal arms selected, and define its complement as $L_I = [K] \setminus I$. We use $\mathcal{T}_{(I)}$ to represent the set of time steps during which exactly $|I|$ agents are active.

For each agent $j$ and arm $k$, let $N_k^j(T) := \sum_{t = 1}^{T} \mathbf{1}(A^j(t) = k)$ be the number of times agent $j$ selects arm $k$ up to time $T$. Summing over all agents, we define $N_k(T):= \sum_{m = 1}^{M} N_k^m(T)$ as the total number of times arm $k$ is selected by all agents up to time $T$.

In addition, let $\mathcal{C}_k(T)$ denote the cumulative number of collisions that occur on arm $k$ by time $T$, defined as:
\[
\mathcal{C}_k(T) := \sum_{t = 1}^{T} \sum_{j = 1}^{M} \mathbf{1}(C^j(t)) \cdot \mathbf{1}(A^j(t) = k).
\]

The regret $R_T$ can be decomposed as
\small
\[
\begin{aligned}
&\sum_{I \in \mathcal{I}} \Bigg(
    \sum_{k \in I} \left(\mu(k) - \mu(I^*)\right) \big(\mathcal{T}(I) - E_{\mu}[N_k(\mathcal{T}_{(I)})]\big)\\
    &\quad + \sum_{k \in L_I} \left(\mu(I^*) - \mu(k^*)\right) E_{\mu}[N_k(\mathcal{T}_I)] \\
&\quad + \sum_{k=1}^K \mu_k E_{\mu}[C_k(\mathcal{T}_{(I)})]
\Bigg)
+ \sum_{k \in L_{I_{\mathrm{min}}}} \left(\mu(k^*) - \mu(k)\right) E_{\mu}[N_k(T)] \\
= & \sum_{I \in \mathcal{I}} \Bigg(
    \sum_{k \in I} \left(\mu(k) - \mu(I^*)\right) \big(\mathcal{T}(I) - E_{\mu}[N_k(\mathcal{T}_{(I)})]\big)
    \\&+ \sum_{k=1}^K \mu(k) E_{\mu}[C_k(\mathcal{T}_{(I)})]
\Bigg) \\
& + \sum_{I \in \mathcal{I}} \sum_{k \in L_I} \left(\mu(I^*) - \mu(k^*)\right) E_{\mu}[N_k(\mathcal{T}_I)]\\ 
& + \sum_{k \in L_{I_{\mathrm{min}}}} \left(\mu(k^*) - \mu(k)\right) E_{\mu}[N_k(T)]
\end{aligned}
\]

The first term can be proved to be greater than or equal to 0. So
\[
R_T \ge \sum_{k\in L_{I_{\text{min}}}}(\mu(k^*)-\mu(k))E_{\mu}[N_k(T)] 
\]

Let $\mathcal{M}$ be a set of distributions with finite means, and let $\mu: \mathcal{M} \rightarrow \mathbb{R}$ be the function that maps $P \in \mathcal{M}$ to its mean. Let $\mu(k^*) \in \mathbb{R}$ and $P \in \mathcal{M}$ have $\mu(P)<\mu(k^*)$ and define
\[
d_{\mathrm{inf}}\left(P, \mu(k^*), \mathcal{M}\right)=\inf _{P^{\prime} \in \mathcal{M}}\left\{\mathrm{D}\left(P, P^{\prime}\right): \mu\left(P^{\prime}\right)>\mu(k^*)\right\}
\]

Let $\mathcal{E}=\mathcal{M}_{1} \times \cdots \times \mathcal{M}_{k}$ and $\pi \in \Pi_{\text {cons }}(\mathcal{E})$ be a consistent policy over $\mathcal{E}$. Then, for all $\nu=\left(P_{i}\right)_{i=1}^{k} \in \mathcal{E}$, it holds that
\[
\liminf _{T \rightarrow \infty} \frac{R_{T}}{\log (\sum\omega T)} \geq c^{*}(\nu, \mathcal{E})=\sum_{i: \hat{\Delta}(i)>0} \frac{\hat{\Delta}(i)}{d_{\mathrm{inf}}\left(P_{i}, \mu(k_i^*), \mathcal{M}_{i}\right)}
\]

Let $d_{i}=d_{\mathrm{inf}}\left(P_{i}, \mu(k_i^*), \mathcal{M}_{i}\right)$. The result will follow from Lemma 4.5, and by showing that for any suboptimal arm $i$, it holds that
\[
\liminf _{T \rightarrow \infty} \frac{\mathbb{E}_{\nu \pi}\left[N_{i}(T)\right]}{\log (T)} \geq \frac{1}{d_{i}}
\]
Fix a suboptimal arm $i$, and let $\varepsilon>0$ be arbitrary and $\nu^{\prime} = \left(P_{j}^{\prime}\right)_{j=1}^{k} \in \mathcal{E}$ be a bandit with $P_{j}^{\prime} = P_{j}$ for $j \neq i$ and $P_{i}^{\prime} \in \mathcal{M}_{i}$ be such that $\mathrm{D}\left(P_{i}, P_{i}^{\prime}\right) \leq d_{i}+\varepsilon$ and $\mu\left(P_{i}^{\prime}\right) > \mu^{*}$, which exists by the definition of $d_{i}$. Let $\mu^{\prime} \in \mathbb{R}^{k}$ be the vector of means of distributions of $\nu^{\prime}$. By Lemma 15.1, we have
\[
\mathrm{D}\left(\mathbb{P}_{\nu \pi}, \mathbb{P}_{\nu^{\prime} \pi}\right) \leq \mathbb{E}_{\nu \pi}\left[N_{i}(T)\right]\left(d_{i}+\varepsilon\right),
\]
and by Theorem 14.2, for any event $A$,
\begin{align*}
    \mathbb{P}_{\nu \pi}(A)+\mathbb{P}_{\nu^{\prime} \pi}\left(A^{c}\right) &\geq \frac{1}{2} \exp \left(-\mathrm{D}\left(\mathbb{P}_{\nu \pi}, \mathbb{P}_{\nu^{\prime} \pi}\right)\right) \\
&\geq \frac{1}{2} \exp \left(-\mathbb{E}_{\nu \pi}\left[N_{i}(T)\right]\left(d_{i}+\varepsilon\right)\right).
\end{align*}

Now choose $A=\left\{N_{i}(T)> \sum\omega T / 2\right\}$, and let $R_{T} = R_{T}(\pi, \nu)$ and $R_{T}^{\prime} = R_{T}\left(\pi, \nu^{\prime}\right)$. Then,
\begin{align*}
& R_{T}+R_{T}^{\prime} \geq \frac{\sum\omega T}{2}\left(\mathbb{P}_{\nu \pi}(A) \hat{\Delta}(i) + \mathbb{P}_{\nu^{\prime} \pi}\left(A^{c}\right)\left(\mu_{i}^{\prime}-\mu_{k_i^*} \right)\right) \\
& \geq \frac{\sum\omega T}{2} \min \left\{\hat{\Delta}(i), \mu_{i}^{\prime}-\mu_{k_i^*} \right\}\left(\mathbb{P}_{\nu \pi}(A)+\mathbb{P}_{\nu^{\prime} \pi}\left(A^{c}\right)\right) \\
& \geq \frac{\sum\omega T}{4} \min \left\{\hat{\Delta}(i), \mu_{i}^{\prime}-\mu_{k_i^*} \right\} \exp \left(-\mathbb{E}_{\nu \pi}\left[N_{i}(T)\right]\left(d_{i}+\varepsilon\right)\right)
\end{align*}

Rearranging and taking the limit inferior leads to
\begin{align*}
&\liminf _{T \rightarrow \infty} \frac{\mathbb{E}_{\nu \pi}\left[N_{i}(T)\right]}{\log (T)}\\ & \geq \frac{1}{d_{i}+\varepsilon} \liminf _{T \rightarrow \infty} \frac{\log \left(\frac{\sum\omega T \min \left\{\hat{\Delta}(i), \mu(I)^{\prime}-\mu(k_i^*)\right\}}{4\left(R_{T}+R_{T}^{\prime}\right)}\right)}{\log ( T)} \\
& =\frac{1}{d_{i}+\varepsilon}\left(1-\limsup _{T \rightarrow \infty} \frac{\log \left(R_{T}+R_{T}^{\prime}\right)}{\log ( T)}\right) =\frac{1}{d_{i}+\varepsilon}
\end{align*}

Finally, we get the following regret:
We use the \textit{KL-Divergence} for any two i.i.d. distributions to present the lower bound. We consider the case where each \(P_{i(\ell)}\) follows a normal distribution \(\mathcal{N}(\mu(i),1)\) for \(i \in [M], \ell \in [L]\). Then,
\[
d_{\inf} \left( P_{i}, \mu(i^*), \mathcal{M} \right) = \hat{\Delta}(i)^2
\]

Hence, the first term in Theorem 2 can be rewritten as
\[
\sum_{i: \hat{\Delta}(i)>0} \frac{1}{\hat{\Delta}(i)}
\]

Hence, we can obtain the final lower bound
\[
\liminf _{T \rightarrow \infty} \frac{R_{T}(\pi,v)}{\log (T)} \ge \sum_{i: \hat{\Delta}(i)>0} \frac{1}{\hat{\Delta}(i)}
\]
\end{proof}

\begin{lemma}
    Assume $M$ agents independently sample arms from an i.i.d. reward process with unknown mean $\mu(k)$. Let $n_t^{(m)}(k)$ be the number of samples that agent $m$ has collected for arm $k$ by time $t$, and define $T_t(k) = \sum_{m=1}^{M} n_t^{(m)}(k)$. Let $\mathrm{CR}_{[0,1]}(n, \delta)$ denote the confidence radius under Hoeffding’s inequality, and define $\mathrm{ECR}_t(k) = \mathrm{CR}_{[0,1]}(T_t(k), \delta)$, where $\delta \in (0,1)$ is a non-increasing sequence.  Then, for any $t$, with probability at least $1 -\delta$, the following inequality holds:
\[
\left| \hat{\mu}_t^{(m)}(k) - \mu(k) \right| \leq 2\beta \cdot \mathrm{CR}_{[0,1]}(T_t(k), \delta)
\]
\end{lemma}

\begin{proof}
Let $s < t$ denote the most recent communication round. Agent $m$’s updated estimate at time $t$ is:

\begin{align*}
&\left| \hat{\mu}_t^{(m)}(k) - \mu(k) \right| \\
&= \left| 
\frac{ \sum_{m'=1}^{M} n_s^{(m')}(k) \tilde{\mu}^{(m')}(k) + X_t^{(m)}(k) - X_s^{(m)}(k) }
     { T_s(k) + n_t^{(m)}(k) - n_s^{(m)}(k) } 
- \mu(k) \right| \\
&= \Bigg| 
\frac{ \sum_{m'=1}^{M} n_s^{(m')}(k) \left( \tilde{\mu}^{(m')}(k) - \bar{\mu}^{(m')}(k) \right) }{T_s(k) + n_t^{(m)}(k) - n_s^{(m)}(k)} \notag \\
&\quad + \frac{ \sum_{m'=1}^{M} n_s^{(m')}(k) \bar{\mu}^{(m')}(k) + X_t^{(m)}(k) - X_s^{(m)}(k) }{T_s(k) + n_t^{(m)}(k) - n_s^{(m)}(k)} - \mu(k) 
\Bigg| \displaybreak[0] \notag \\
&\leq 
\left| 
\frac{ \sum_{m'=1}^{M} X_s^{(m')}(k) + X_t^{(m)}(k) - X_s^{(m)}(k) }
     { T_s(k) + n_t^{(m)}(k) - n_s^{(m)}(k) } 
- \mu(k) 
\right| \\
&\quad + 
\left| 
\frac{ \sum_{m'=1}^{M} n_s^{(m')}(k) \left( \tilde{\mu}^{(m')}(k) - \bar{\mu}^{(m')}(k) \right) }
     { T_s(k) + n_t^{(m)}(k) - n_s^{(m)}(k) } 
\right| \\
&\leq 
\mathrm{CR}_{[0,1]}\left( T_s(k) + n_t^{(m)}(k) - n_s^{(m)}(k), \delta \right) 
+ \sqrt{ \frac{1}{T_t(k)} } \\
&\stackrel{(a)}{\leq} 
\mathrm{CR}_{[0,1]}\left( T_s(k), \delta \right) 
+ \sqrt{ \frac{1}{T_t(k)} } \\
&\stackrel{(b)}{\leq} 
2\beta \cdot \mathrm{CR}_{[0,1]}\left( T_t(k), \delta \right)
\end{align*}

where (a) uses the fact that $\mathrm{CR}$ increases as the number of samples decreases, and (b) follows from the assumption that the communication condition is not met at time $t$.
\end{proof}

\begin{lemma}
\label{lemma:2_N_e}
 With probability at least $1-\delta$, every optimal arm k is accepted
after at most $\left(\frac{32\beta^2\log\delta^{-1}}{\Delta_\text{min}^2}\right )$pulls during exploration phases, and every sub-optimal arm k is
rejected after at most $\left (\frac{32\beta^2\log\delta^{-1}}{(\mu(M)-\mu(k))^2}\right )$pulls during exploration phases.\\
\end{lemma}
\begin{proof} 
If the estimation of the mean of reward lies in the 
confidence interval, for any active optimal arm k, we have the following:
\begin{align*}
    2(2\beta \mathrm{CR}_{[0,1]}(T_{t}(k),\delta)+ 2(2\beta \mathrm{CR}_{[0,1]}(T_{t}(k),\delta) \ge \Delta(k)
\end{align*}
$\Delta(k) = \Delta_{\text{min}}$ is the smallest gap among all optimal arms. Otherwise, the optimal arms will be correctly sorted and accepted.
Then we have:
    \[4(2\beta\mathrm{CR}_{[0,1]}(T_v(k),\delta) \ge \Delta(k)\]
Since in the algorithm, each active agent pulls arm k once in each round without collision, we can prove that the number of pullings before arm i is accepted is upper  bounded by :  
\begin{align*}
     T_v(k) &\le \frac{32\beta^2\log \delta^{-1}}{\Delta_{\text{min}}^2}
\end{align*}
so the optimal active arm k will be accepted after at most $(\frac{32\beta^2\log \delta^{-1}}{\Delta_{\text{min}}^2} )$ pulls and the analysis about when suboptimal arms will be rejected is similar to the one above and we will have $T_v(k)=\left ( \frac {32\beta^2\log \delta^{-1}}{\Delta(k)^2}\right )$,$\Delta(k)=\mu(M)-\mu(k)$.\\
\end{proof}

To characterize the group regret, we begin by decomposing the difference between the cumulative rewards obtained by the optimal matching and those obtained by the agents' joint decisions. Recall that at each time step \( t \), the agents select a subset \( K(t) \subset [K] \), while the optimal set is denoted by \( \mathcal{I}(t) \). Let \( \mu(k) \) denote the mean reward of arm \( k \). Let \(I = |\mathcal{I}(t)|\) denote the number of active agents. \(T(I)\) is the total number of time steps during which exactly \(I\) agents are active. For each arm \(k \in [K]\), \(N(I,k)\) is the total number of times arm \(k\) is pulled during these time steps.
The group regret can be written as follows:

\begin{align*}
    &\sum_{t=1}^{T}\left(\sum_{k \in [\mathcal{I}(t)]} \mu(k)
    -\sum_{k \in [K(t)]} \mu(k)\right) \\
    &\overset{(a)}{=} \sum_{I\in \mathcal{I}} \Bigg[
        \sum_{k\leq I} (\mu(k)-\mu(I))(T(I)-N(I,k)) \\
    &\hspace{5em} + \sum_{k>I} (\mu(I)-\mu(k))N(I,k)
    \Bigg] \\
    &\overset{(b)}{=} \sum_{I\in \mathcal{I}} \Bigg[
        \underbrace{
            \sum_{k\leq I} (\mu(k)-\mu(I))(T(I)-N(I,k))
        }_{\text{(1)}} \\
    &
        + \underbrace{
            \sum_{I<k\leq M} (\mu(I)-\mu(k))N(I,k)
            + \sum_{k>M} (\mu(I)-\mu(M))N(I,k)
        }_{\text{(2)}}
    \Bigg] \\
    &\quad + \underbrace{
        \sum_{k>M} (\mu(M)-\mu(k))N(k)
    }_{\text{(3)}}
\end{align*}
s
In the above decomposition:
(a): When considering the regret, we still analyse it by reformulating the problem into different MMAB instances. For each instance, which is defined by the number of active agents, we apply a similar decomposition as in the analysis of SynCD. 
(b): We then perform algebraic manipulations to split the expression into three parts, labeled as (1), (2), and (3), which will be analyzed separately in the following sections.
Before starting to prove the upper bound for each part. We first define \(\gamma(I)\), which denotes the frequency at which exactly \(I\) agents are activated during an exploration phase. Then, we introduce \(A(I)\), a coefficient representing the proportion of arm pulls that occur when \(I\) agents are simultaneously active.
\[
A(I) = \frac{\gamma(I)\times M_I}{\sum_{J\in \mathcal{I}}{\gamma(J)\times M_J}}
\] 
What's more, we define the largest gap among the top M arms as $\Delta$, $P$ as the total number of communication rounds, and $C_p$ as the number of exploration phase between the pth and (p-1)th communication phase. We denote $T(I,p)$ as the total times of pulls when $I$ agents activate and $T(p)$ as the total times of pulls at the p-th communication round.
\begin{lemma}
    For part(1), it can be upper bounded as 
    \[
     A(I) \cdot \frac{\Delta}{\Delta_{\min}} \left(\sum_{I< k\le M}\frac{b\log \delta^{-1}}{\Delta_{\text{min}}}+(1+ \beta)  \cdot \sum_{j > M} \frac{c \log \delta^{-1}}{\Delta_j}\right)
    \]
\end{lemma}
\begin{proof}
We first divided the regret into two parts.
\begin{align*}
    & \sum_{k \leq I} (\mu(k) - \mu(I)) \bigl(T(I) - N(I,k)\bigr) \\
    &\overset{(a)}{=} \sum_{k \leq I} (\mu(k) - \mu(I)) 
    \sum_{p=1}^P C_p \gamma(I) \cdot \text{lcm} \cdot (K_p - M_I) \\
    &= \sum_{k \leq I} (\mu(k) - \mu(I)) 
    \sum_{p=1}^P C_p \gamma(I) \cdot \text{lcm} \cdot (M - M_I) \\
    &\quad + \sum_{p=1}^P C_p (\mu(k) - \mu(I)) \gamma(I) 
    \cdot \text{lcm} \cdot (K_p - M) \\
    &= \frac{M - M_I}{M_I} 
    \sum_{k \leq I} (\mu(k) - \mu(I)) N(I,k) \\
    &\quad + \sum_{k \leq I} (\mu(k) - \mu(I)) 
    \sum_{p=1}^P C_p \gamma(I) \cdot \text{lcm} \cdot 
    \sum_{j > M} \mathbf{1}_{\hat{t}(j) > T(p-1)} \\
    &= \underbrace{
        A(I) \cdot \frac{M - M_I}{M_I} 
        \sum_{k \leq I} (\mu(k) - \mu(I)) N(k)
    }_{(a)} \\
    &\quad + \underbrace{
        \sum_{k \leq I} (\mu(k) - \mu(I)) 
        \sum_{p=1}^P C_p \gamma(I) \cdot \text{lcm} \cdot 
        \sum_{j > M} \mathbf{1}_{\hat{t}(j) > T(p-1)}
    }_{\text{(b)}}
\end{align*}

From Lemma ~\ref{lemma:2_N_e}, for simplicity, we omit constant factors and denote them by a universal constant \(b\) when necessary.
We can get that (a) can be upper bounded as:
\[
A(I)\frac{M - M_I}{M_I} \sum_{k\le I}\frac{\Delta}{\Delta_{\text{min}}}\frac{b\log \delta^{-1}}{\Delta_{\text{min}}}
\]
because the gaps are uniformly scaled to $\Delta$, it actually equals 
\[
    A(I) \sum_{I< k\le M}\frac{\Delta}{\Delta_{\text{min}}}\frac{b\log \delta^{-1}}{\Delta_{\text{min}}}
\]

Then, for the (b) part, we conduct an analysis similar to that used in the proof of SynCD. 
\begin{align*}
&\sum_{k \leq I} (\mu(k) - \mu(I)) \sum_{p=1}^P C_p \gamma(I) \cdot \text{lcm} \cdot \sum_{j > M} \mathbf{1}_{\hat{t}(j) > T(p-1)} \\
&\leq \sum_{j > M} \sum_{k \leq I} (\mu(k) - \mu(I)) \sum_{p=1}^{P_j} C_p \gamma(I) \cdot \text{lcm} \\
&\leq \sum_{j > M} \sum_{k \leq I} \Delta \sum_{p=1}^{P_j} C_p \gamma(I) \cdot \text{lcm} \\
&= \sum_{j > M} \Delta \cdot \sum_{p=1}^{P_j}C_p \gamma(I) \cdot \text{lcm} \cdot M_I \\
&= \sum_{j > M} \Delta \cdot \sum_{p=1}^{P_j}\left( T(I, p) - T(I, p-1) \right) \\
&= A(I) \cdot \sum_{j > M} \Delta \cdot \sum_{p=1}^{P_j}\left( T(p) - T(p-1) \right) \\
&= A(I) \cdot \sum_{j > M} \sum_{p=1}^{P_j} \frac{\Delta}{\Delta(p)} \cdot \Delta(p) \cdot \left( T(p) - T(p-1) \right) \\
&\leq A(I) \cdot \frac{\Delta}{\Delta_{\min}} \cdot \sum_{j > M} \sum_{p=1}^{P_j} \Delta(p) \cdot \left( T(p) - T(p-1) \right) \\
&\leq A(I) \cdot \frac{\Delta}{\Delta_{\min}} \cdot \sum_{j > M} \sum_{p=1}^{P_j} 
\Delta(p)\\ 
&\cdot \left( \sqrt{T(p)} + \sqrt{T(p-1)} \right) \left( \sqrt{T(p)} - \sqrt{T(p-1)} \right) \\
&= A(I) \cdot \frac{\Delta}{\Delta_{\min}} \cdot \sqrt{c \log \delta^{-1}} \\
&\cdot 
\sum_{j > M} \sum_{p=1}^{N_j} \left( 1 + \sqrt{ \frac{T(p-1)}{T(p)} } \right) \cdot 
\left( \sqrt{T(p)} - \sqrt{T(p-1)} \right) \\
&\leq (1 + \beta) A(I) \cdot \frac{\Delta}{\Delta_{\min}} \cdot \sqrt{c \log \delta^{-1}} \cdot \sum_{j > M} \sqrt{T_{N_j}} \\
&\leq (1 + \beta) A(I) \cdot \frac{\Delta}{\Delta_{\min}} \cdot \sum_{j > M} \frac{c \log \delta^{-1}}{\Delta_j}
\end{align*}
(a)+(b), then we can upper bound (1) as:
\[
      A(I) \cdot \frac{\Delta}{\Delta_{\min}} \left(\sum_{I< k\le M}\frac{b\log \delta^{-1}}{\Delta_{\text{min}}}+(1+ \beta)  \cdot \sum_{j > M} \frac{c \log \delta^{-1}}{\Delta_j}\right)
\]
\end{proof}

\begin{lemma}
    For part (2), it can be upper bounded as
    \[
    A(I) \left( \sum_{I<k\leq M} \frac{\Delta}{\Delta_{\min}} \cdot \frac{c\log(\delta^{-1})}{\Delta_{\min}}
+ \sum_{k>M} \frac{\Delta}{\Delta(k)} \cdot \frac{b\log(\delta^{-1})}{\Delta(k)} \right)
    \]
\end{lemma}
\begin{proof}
For simplicity and considering the difficulty of deriving a precise bound for our algorithm, we apply a relatively loose relaxation in our analysis.
\begin{align*}
&\sum_{I<k\leq M} (\mu(I)-\mu(k))N(I,k) + \sum_{k>M} (\mu(I)-\mu(M))N(I,k) \\
\le& \sum_{I<k\leq M}\Delta N(I,k) + \sum_{k>M} \Delta N(I,k) \\
=& A(I) \left( \sum_{I<k\leq M} \frac{\Delta}{\Delta_{\min}} \cdot \Delta_{\min} N(k)
+ \sum_{k> M} \frac{\Delta}{\Delta(k)} \cdot \Delta(k) N(k) \right) \\
\leq &A(I) \left( \sum_{I<k\leq M} \frac{\Delta}{\Delta_{\min}} \cdot \frac{b\log(\delta^{-1})}{\Delta_{\min}}
+ \sum_{k>M} \frac{\Delta}{\Delta(k)} \cdot \frac{c\log(\delta^{-1})}{\Delta(k)} \right) \\
\end{align*}
Here, we use Lemma~\ref{lemma:2_N_e} to bound the total number of pulls of arm $k$,$N(k)$.
\end{proof}

\begin{lemma}
    For part (3), it can be upper bounded as
    \[
    \sum_{k > M} \frac{c \log \delta^{-1}}{\Delta_k}
    \]
\end{lemma}
This Lemma can be easily get by using Lemma~\ref{lemma:2_N_e}.

In conclusion, We sum the 3 parts together and we can get the final regret.
\begin{align*}
R[T]&\leq \sum_{I} A(I) \left( 
    \frac{\Delta}{\Delta_{\min}} \sum_{I < k \leq M} \frac{2b \log \delta^{-1}}{\Delta_{\min}} 
    + \sum_{k > M} \left( \frac{\Delta}{\Delta_k} + \frac{(1+\beta)\Delta}{\Delta_{\min}} \right) \cdot \frac{c \log \delta^{-1}}{\Delta_k} 
\right) \\
&\quad + \sum_{k > M} \frac{c \log \delta^{-1}}{\Delta_k}
\end{align*}

Next, we will prove the communication cost of the extended algorithm,
\begin{lemma}[Communication Rounds Upper Bound]
\label{lemma:comm_round_e}
The total number of communication rounds can be upper bounded by:
\begin{align*}
    \sum_{k\le M}\log_{\beta}(\frac{8\beta}{\Delta_{\text{min}}})+\sum_{k>M}\log_{\beta}(\frac{8\beta}{(\mu(M)-\mu(k))})
\end{align*}
\end{lemma}
\begin{proof}
Let $\tau_k$ be the last round to pull the Top-$M$ optimal arms, after which they will all be accepted. In our algorithm, this round is essentially determined by the minimal gap $\Delta_{\min}$.
\begin{align*}
    4(2\beta\text{CR}_{[0.1]}(T_{\tau_k},\delta)) \ge \Delta_{\text{min}}
\end{align*}
And we can get:
\begin{align*}
    8\beta\text{ECR}_{\tau_k}(k)\ge \Delta_{\text{min}}
\end{align*}
Because of the check condition, we can evaluate the times to communicate as follows:
\begin{align*}
    \log_{\beta}(\frac{ECR_1(k)}{ECR_{\tau_k}(k)}) \le \log_{\beta}(\frac{8\beta}{\Delta_{\text{min}}})
\end{align*}
So the expected number of communications by optimal arms is at most 
\begin{align*}
    \sum_{k\le M}\log_{\beta}(\frac{8\beta}{\Delta_{\text{min}}})
\end{align*}
the analysis is same on suboptimal arms $\displaystyle \sum_{k>M}\log_{\beta}(\frac{8\beta}{(\mu(M)-\mu(k))})$ and sum up, the total number of communication rounds overheads is upper bounded by
\begin{align*}
    \sum_{k\le M}\log_{\beta}(\frac{8\beta}{\Delta_{\text{min}}})+\sum_{k>M}\log_{\beta}(\frac{8\beta}{(\mu(M)-\mu(k))})
\end{align*}
\end{proof}
\begin{lemma}[Bit Complexity for Quantized Message]
    We simplify this expression $ \frac{\omega_{(m)}}{\sum_{i=1}^{M}\omega_i} $ as $\alpha_m$ By applying the adapting quantization strategy, in expectation, the ($7+ \log_2(1+\beta+\sqrt{\alpha_m\ln 2/2}$) bits are sufficient to represent the truncated version of $\tilde{\delta}^{(m)}_{t}(k)$
\end{lemma}

\begin{proof}
the quantization leads to a quantization error of at most $\sqrt{\frac{1}{T_{t}(k)}}$ and have the following holds:
\begin{align*}
    \left| \tilde{\mu}_{t}^{(m)}(k)- \bar{\mu}_{t}^{(m)}(k)  \right| \le \sqrt{\frac{1}{T_{t}(k)}}
\end{align*}
We set the truncated version of the difference $\tilde{\delta}^{(m)}_{t}(k)$ and it can be bounded as  
\begin{align*}
    \left| \tilde{\delta}^{(m)}_{t}(k) \right| 
    &=  \left| \tilde{\mu}_{t}^{(m)}(k)-\tilde{\mu}_{s}^{(m)}(k)\right|\\
    &= \left| \tilde{\mu}_{t}^{(m)}(k)-\bar{\mu}_{t}^{(m)}(k)
        -(\tilde{\mu}_{s}^{(m)}(k)-\bar{\mu}_{s}^{(m)}(k)) \right.\\
    &\quad \left. + \bar{\mu}_{t}^{(m)}(k)-\bar{\mu}_{s}^{(m)}(k)\right|\\
    &\le \left| \tilde{\mu}_{t}^{(m)}(k)-\bar{\mu}_{t}^{(m)}(k)\right| 
        + \left| \tilde{\mu}_{s}^{(m)}(k)-\bar{\mu}_{s}^{(m)}(k)\right| \\
    &\quad + \left |\bar{\mu}_{t}^{(m)}(k)-\bar{\mu}_{s}^{(m)}(k)\right |\\
    &\le \sqrt{\frac{1}{T_{t}(k)}}+\sqrt{\frac{1}{T_{s}(k)}} 
        + \left |\bar{\mu}_{t}^{(m)}(k)-\bar{\mu}_{s}^{(m)}(k)\right |
\end{align*}

Further, we consider $\bar{\mu}_{t}^{(m)}(k)-\bar{\mu}_{s}^{(m)}(k)$,we can rewrite it as 
\begin{align*}
   & \bar{\mu}_{t}^{(m)}(k)-\bar{\mu}_{s}^{(m)}(k)
    =\frac{\sum^{t}_{\text{p=1}}\gamma^{(m)}_{\text{p}}(k)}{n_{t}(k)}- \frac{ \sum^{s}_{\text{p=1}}\gamma^{(m)}_{\text{p}}(k)}{n_{s}(k)}\\
    &=\frac{ \sum^{s}_{\text{p=1}}\gamma^{(m)}_{\text{p}}(k)+\sum^{t}_{\text{p=s}}\gamma_{s}^{(m)}(k)}{n_{t}(k)}- \frac{ \sum^s_{\text{p=1}}\gamma^{(m)}_{\text{p}}(k)}{n_{s}(k)}\\
    &=\frac{1}{n_{t}(k)}\sum^{t}_{\text{p=s}}\gamma_{s}^{(m)}(k)+(\frac{1}{n_{t}(k)}-\frac{1}{n_{s}(k)})\sum_{\text{p=1}}^{s}\gamma_{s}^{(m)}(k)
\end{align*}
which is a $\frac{1}{2}\sqrt{\frac{1}{n_{s}(k)}-\frac{1}{n_{t}(k)}}$ sub-Gussian random variable since the unility samples are independent across time. Thus, we can further derive that, with a dummy variable x,
\begin{align*}
    &\mathbb{P}\left(
        \left |\bar{\mu}_{t}^{(m)}(k)-\bar{\mu}_{s}^{(m)}(k)\right |
        \ge \sqrt{\frac{x^2}{T_{t}(k)}}
    \right) \\
    &\le 2\exp\left[
        -\frac{2x^2}{T_{t}(k)} \cdot 
        \frac{1}{\frac{1}{n_{s}(k)} - \frac{1}{n_{t}(k)}}
    \right] \\
    \Rightarrow\quad 
    &\mathbb{P}\left(
        \left |\bar{\mu}_{t}^{(m)}(k)-\bar{\mu}_{s}^{(m)}(k)\right |
        \ge \sqrt{\frac{x^2}{T_{t}(k)}}
    \right) \\
    &\le 2\exp\left[
        -2x^2 \cdot \frac{\omega^{(m)}}{
        \sum_{i=1}^{M} \omega_i \cdot \beta^2}
    \right] \\
    \Rightarrow\quad 
    &\mathbb{P}\left( 
        \left| \tilde{\delta}^{(m)}_{t}(k) \right| 
        \ge \sqrt{\frac{1}{T_{t}(k)}} + 
        \sqrt{\frac{1}{T_{s}(k)}} + 
        \sqrt{\frac{x^2}{T_{t}(k)}}
    \right) \\
    &\le 2\exp\left[
        -2x^2 \cdot \frac{\omega^{(m)}}{
        \sum_{i=1}^{M} \omega_i \cdot \beta^2}
    \right] \\
    \overset{(a)}{\Rightarrow}\quad 
    &\mathbb{P}\left(
        L_{t}^{(m)}(k) > 
        3 + \frac{\log T_{t}(k)}{2} + 
        \log_2\left(\frac{1+\beta+x}{\sqrt{T_{t}(k)}}\right)
    \right) \\
    &\le 2\exp\left[
        -2x^2 \cdot \frac{\omega^{(m)}}{
        \sum_{i=1}^{M} \omega_i \cdot \beta^2}
    \right] \\
    \Rightarrow\quad 
    &\mathbb{P}\left(
        L_{t}^{(m)}(k) \le 3 + 
        \log_2(1+\beta+x)
    \right) \\
    &\ge 1 - 2\exp\left[
        -2x^2 \cdot \frac{\omega^{(m)}}{
        \sum_{i=1}^{M} \omega_i \cdot \beta^2}
    \right] \\
    \overset{(b)}{\Rightarrow}\quad 
    &\mathbb{P}\left(
        L_{t}^{(m)}(k) \le l
    \right) \\
    &\ge 1 - 2\exp\left[
        -2x^2 \cdot \frac{\omega^{(m)}}{
        \sum_{i=1}^{M} \omega_i \cdot \beta^2}
    \right]
\end{align*}

where $L^r_{k,m}$ in implication (a) is the length of the truncated version $|\tilde{\delta}^r_{k,m}|$ and is upper bounded by 
\begin{align*}
     L_{t}^{(m)}(k)&\leq \lceil 1+\log T_{t}(k)/2\rceil-\lfloor\log_2(1/|\tilde{\delta}_{t}^{(m)}(k)|)\rfloor\\
     &\leq 3 + \frac{\log{T_{t}}(k)}{2}+\log_2(|\tilde{\delta}_{t}^{(m)}(k)|).
 \end{align*}  
 In deriving (b), we substitute the variable $3+\log_2(1+\beta+x)$ with $l$, which satisfies that $l\geq 1+\beta+\log_2(3+\beta\sqrt{\alpha_m\ln 2/2})$, we simplify this expression $ \frac{\omega^{(m)}}{\sum_{i=1}^{M}\omega_i} $ as $\alpha_m$ and thus equivalently $x = 2^{l-3}-1-\beta$. With the above results and viewing $L_{t}^{(m)}(k)$ as a random variable, we have that its cumulative distribution function (CDF) $F_{L_{t}^{(m)}(k)}(l)$ satisfies the following property: $\forall l > b$,where $b = \lceil3+\log_2(1+\beta+\beta\sqrt{\alpha_m\ln 2/2})\rceil$
\begin{align*}
F_{L_{t}^{(m)}(k)}(l) = \mathbb{P}\left(L_{r}^{(m)}(k) \leq l\right)\geq 1-2\exp\left[-\frac{(2^{l-3}-1-\beta)^2}{\alpha_m\cdot \beta^2}\right].
\end{align*}
Using the property of CDF, we can bound the expectation of $L_{t}^{(m)}(k)$ as
\begin{align*}
    \mathbb{E}\left[L_{t}^{(m)}(k)\right] &= \sum_{l = 0}^{\infty} (1-F_{L^r_{k,m}}(l))\\
    &\leq b+ \sum_{l=b}^{\infty}2\exp\left[-\frac{(2^{l-3}-1-\beta)^2}{\alpha_m\cdot \beta^2}\right]\\
    &\leq b+ \int_{l=b}^{\infty}2\exp\left[-\frac{(2^{l-3}-1-\beta)^2}{\alpha_m\cdot \beta^2}\right]\mathrm{d} l\\
    &\leq b+2\\
    &\leq 6+ \log_2(1+\beta+\beta\sqrt{\alpha_m\ln 2/2})
\end{align*}
Thus, we have that in expectation, the truncated version of $\tilde{\delta}^{(m)}_{t}(k)$ has a length that is less than $ 6+ \log_2(1+\beta+\beta\sqrt{\alpha_m\ln 2/2})$ bits.. In addition, 1-bit information should also be transmitted to indicate the sign of $\tilde{\delta}^{(m)}_{t}(k)$.As a
summary, in expectation, $ 7+ \log_2(1+\beta+\beta\sqrt{\alpha_m\ln 2/2})$ bits is sufficient to represent the truncated version of $\tilde{\delta}^{(m)}_{t}(k)$
\end{proof}
\begin{lemma}
The number of bits required to transmit the first estimated mean between agents in our differential quantization scheme is upper bounded by
\[
    \left\lceil 1 + \frac{1}{2} \log \left( \frac{\log(\delta^{-1})}{2 \beta^{2}} + M K \right) \right\rceil,
\]
which implies an overall communication cost of order \( \mathcal{O}(\log \log T) \).
\end{lemma}

\begin{proof}
In our threshold-based triggering mechanism, we initialize the Estimated Confidence Radius (ECR) to 1. A communication is triggered whenever the current ECR satisfies 
\[
ECR_t \leq \beta \times ECR_{\mathrm{last}},
\]
where \( \beta \in (0,1) \) is a fixed threshold parameter. According to Lemma~\ref{lemma:comm_round}, this guarantees a constant number of communications.

At the time of the first communication, by the definition of ECR, we have
\[
\sqrt{\frac{\log(\delta^{-1})}{2 T_t(k)}} \leq \beta \times 1,
\]
which can be rearranged to give
\[
T_t(k) \leq \frac{\log(\delta^{-1})}{2 \beta^{2}} + M K.
\]

In our differential quantization strategy, the mean reward is quantized using 
\[
\left\lceil 1 + \frac{\log T_t(k)}{2} \right\rceil
\]
bits. Since this is the first communication, there is no previous mean estimate available for differential encoding, so the full quantized mean must be transmitted.

Substituting the bound on \( T_t(k) \), we have that the number of bits for the first communication is bounded by
\[
\left\lceil 1 + \frac{1}{2} \log \left( \frac{\log(\delta^{-1})}{2 \beta^{2}} + M K \right) \right\rceil,
\]
which grows asymptotically as \( \mathcal{O}(\log \log T) \) when \( T \to \infty \).
\end{proof}

Finally, communication only occurs at every least common multiple (LCM) period $l$, when all agents are simultaneously activated. Combining this with the previous results, we obtain the final bound/result as follows:
Let \(\alpha_m = \frac{\omega^{(m)}}{\sum_{i=1}^M \omega_i}\). The expected communication cost is bounded by
\[
\begin{aligned}
\mathbb{E}[R_C] = O\Bigg(&
2M^2 l\sum_{m=1}^{M}\left(\sum_{k=1}^K \log_{\beta}\left(\frac{8 \beta}{\Delta(k)}\right)-1\right)
C \\
& + 2M^3  Kl \\
& + 2M^3 l\cdot \left(
1 + \frac{1}{2} \log \left( \frac{\log(\delta^{-1})}{2 \beta^{2}} + M K \right)
\right)
\Bigg).
\end{aligned}\\
\]
\[\text{where} \quad C := \left[
 7 + \log_2\left(1 + \beta + \beta \sqrt{\frac{\alpha_m \ln 2}{2}}\right)
\right]
\]

\bibliography{ref}

\bibliography{algo}

\begin{thebibliography}{33}
\providecommand{\natexlab}[1]{#1}

\bibitem[{Avner and Mannor(2014)}]{avner2014concurrent}
Avner, O.; and Mannor, S. 2014.
\newblock Concurrent multi-armed bandits.
\newblock \emph{arXiv preprint arXiv:1405.3279}.

\bibitem[{Avner and Mannor(2019)}]{Avner2019}
Avner, O.; and Mannor, S. 2019.
\newblock Multi-user communication networks: A coordinated multi-armed bandit approach.
\newblock \emph{IEEE/ACM Transactions on Networking}, 27(6): 2192--2207.

\bibitem[{Bande and Veeravalli(2019)}]{bande2019collaborative}
Bande, N.; and Veeravalli, V.~V. 2019.
\newblock Collaborative multi-player multi-armed bandits in dynamic environments.
\newblock \emph{IEEE Transactions on Information Theory}, 65(6): 3761--3779.

\bibitem[{Besson and Kaufmann(2018{\natexlab{a}})}]{besson2018multi}
Besson, L.; and Kaufmann, E. 2018{\natexlab{a}}.
\newblock Multi-player bandits revisited.
\newblock In \emph{Algorithmic Learning Theory}, 56--92. PMLR.

\bibitem[{Besson and Kaufmann(2018{\natexlab{b}})}]{Besson2018}
Besson, L.; and Kaufmann, E. 2018{\natexlab{b}}.
\newblock Multi-player bandits revisited.
\newblock In \emph{Algorithmic Learning Theory}, 56--92.

\bibitem[{Bistritz and Bambos(2020)}]{bistritz2020cooperative}
Bistritz, I.; and Bambos, N. 2020.
\newblock Cooperative multi-player bandit optimization.
\newblock \emph{Advances in Neural Information Processing Systems}, 33: 2016--2027.

\bibitem[{Bistritz and Leshem(2020)}]{Bistritz2020}
Bistritz, I.; and Leshem, A. 2020.
\newblock Game of thrones: Fully distributed learning for multiplayer bandits.
\newblock \emph{Mathematics of Operations Research}.

\bibitem[{Bonnefoi et~al.(2017)Bonnefoi, Besson, Moy, Kaufmann, and Palicot}]{bonnefoi2017multi}
Bonnefoi, R.; Besson, L.; Moy, C.; Kaufmann, E.; and Palicot, J. 2017.
\newblock Multi-Armed Bandit Learning in IoT Networks: Learning helps even in non-stationary settings.
\newblock In \emph{International Conference on Cognitive Radio Oriented Wireless Networks}, 173--185. Springer.

\bibitem[{Boursier and Perchet(2019)}]{boursier2019sic}
Boursier, E.; and Perchet, V. 2019.
\newblock SIC-MMAB: Synchronisation involves communication in multiplayer multi-armed bandits.
\newblock \emph{Advances in Neural Information Processing Systems}, 32.

\bibitem[{Chakraborty et~al.(2017)Chakraborty, Chua, Das, and Juba}]{chakraborty2017coordinated}
Chakraborty, M.; Chua, K. Y.~P.; Das, S.; and Juba, B. 2017.
\newblock Coordinated Versus Decentralized Exploration In Multi-Agent Multi-Armed Bandits.
\newblock In \emph{IJCAI}, 164--170.

\bibitem[{Chawla et~al.(2020)Chawla, Sankararaman, Ganesh, and Shakkottai}]{chawla2020gossiping}
Chawla, R.; Sankararaman, A.; Ganesh, A.; and Shakkottai, S. 2020.
\newblock The gossiping insert-eliminate algorithm for multi-agent bandits.
\newblock In \emph{International conference on artificial intelligence and statistics}, 3471--3481. PMLR.

\bibitem[{Chen et~al.(2023)Chen, Yang, Wang, Liu, Hajiesmaili, Lui, and Towsley}]{chen2023demand}
Chen, Y.-Z.~J.; Yang, L.; Wang, X.; Liu, X.; Hajiesmaili, M.; Lui, J.~C.; and Towsley, D. 2023.
\newblock On-demand communication for asynchronous multi-agent bandits.
\newblock In \emph{International Conference on Artificial Intelligence and Statistics}, 3903--3930. PMLR.

\bibitem[{Dakdouk, Lasaulce, and Proutiere(2021)}]{dakdouk2021multi}
Dakdouk, R.; Lasaulce, S.; and Proutiere, A. 2021.
\newblock Multi-player multi-armed bandits with heterogeneous activation rates.
\newblock \emph{IEEE Transactions on Information Theory}, 67(8): 5031--5049.

\bibitem[{Darak and Hanawal(2019{\natexlab{a}})}]{Darak2019}
Darak, S.~J.; and Hanawal, M.~K. 2019{\natexlab{a}}.
\newblock Multi-player multi-armed bandits for stable allocation in heterogeneous ad-hoc networks.
\newblock \emph{IEEE Journal on Selected Areas in Communications}, 37(10): 2350--2363.

\bibitem[{Darak and Hanawal(2019{\natexlab{b}})}]{darak2019multi}
Darak, S.~J.; and Hanawal, M.~K. 2019{\natexlab{b}}.
\newblock Multi-player multi-armed bandits for stable allocation in heterogeneous ad-hoc networks.
\newblock \emph{IEEE Journal on Selected Areas in Communications}, 37(10): 2350--2363.

\bibitem[{F{\'e}raud, Alami, and Laroche(2019)}]{feraud2019decentralized}
F{\'e}raud, R.; Alami, R.; and Laroche, R. 2019.
\newblock Decentralized exploration in multi-armed bandits.
\newblock In \emph{International Conference on Machine Learning}, 1901--1909. PMLR.

\bibitem[{Jouini et~al.(2009)Jouini, Ernst, Moy, and Palicot}]{jouini2009multi}
Jouini, W.; Ernst, D.; Moy, C.; and Palicot, J. 2009.
\newblock Multi-armed bandit based policies for cognitive radio's decision making issues.
\newblock In \emph{2009 3rd International Conference on Signals, Circuits and Systems (SCS)}, 1--6. IEEE.

\bibitem[{Kalathil, Nayyar, and Jain(2014)}]{Kalathil2014}
Kalathil, D.; Nayyar, N.; and Jain, R. 2014.
\newblock Decentralized learning for multiplayer multiarmed bandits.
\newblock \emph{IEEE Transactions on Information Theory}, 60(4): 2331--2345.

\bibitem[{Kolla, Jagannathan, and Gopalan(2018)}]{kolla2018collaborative}
Kolla, R.~K.; Jagannathan, K.; and Gopalan, A. 2018.
\newblock Collaborative learning of stochastic bandits over a social network.
\newblock \emph{IEEE/ACM Transactions on Networking}, 26(4): 1782--1795.

\bibitem[{Kumar et~al.(2010)Kumar, Zarychanski, Light, Parrillo, Maki, Simon, Laporta, Lapinsky, Ellis, Mirzanejad et~al.}]{kumar2010early}
Kumar, A.; Zarychanski, R.; Light, B.; Parrillo, J.; Maki, D.; Simon, D.; Laporta, D.; Lapinsky, S.; Ellis, P.; Mirzanejad, Y.; et~al. 2010.
\newblock Early combination antibiotic therapy yields improved survival compared with monotherapy in septic shock: a propensity-matched analysis.
\newblock \emph{Critical care medicine}, 38(9): 1773--1785.

\bibitem[{Lai and Robbins(1985)}]{lai1985asymptotically}
Lai, T.~L.; and Robbins, H. 1985.
\newblock Asymptotically efficient adaptive allocation rules.
\newblock \emph{Advances in applied mathematics}, 6(1): 4--22.

\bibitem[{Landgren, Srivastava, and Leonard(2016)}]{landgren2016distributed}
Landgren, P.; Srivastava, V.; and Leonard, N.~E. 2016.
\newblock Distributed cooperative decision-making in multiarmed bandits: Frequentist and bayesian algorithms.
\newblock In \emph{2016 IEEE 55th Conference on Decision and Control (CDC)}, 167--172. IEEE.

\bibitem[{Liu and Zhao(2010)}]{liu2010distributed}
Liu, K.; and Zhao, Q. 2010.
\newblock Distributed learning in multi-armed bandit with multiple players.
\newblock \emph{IEEE transactions on signal processing}, 58(11): 5667--5681.

\bibitem[{Mart{\'\i}nez-Rubio, Kanade, and Rebeschini(2019)}]{martinez2019decentralized}
Mart{\'\i}nez-Rubio, D.; Kanade, V.; and Rebeschini, P. 2019.
\newblock Decentralized cooperative stochastic bandits.
\newblock \emph{Advances in Neural Information Processing Systems}, 32.

\bibitem[{Nayyar, Kalathil, and Jain(2016)}]{nayyar2016regret}
Nayyar, N.; Kalathil, D.; and Jain, R. 2016.
\newblock On regret-optimal learning in decentralized multiplayer multiarmed bandits.
\newblock \emph{IEEE Transactions on Control of Network Systems}, 5(1): 597--606.

\bibitem[{Richard, Koolen, and Lasaulce(2023)}]{richard2023asynchronous}
Richard, M.; Koolen, W.~M.; and Lasaulce, S. 2023.
\newblock Asynchronous multi-player multi-armed bandits.
\newblock \emph{arXiv preprint arXiv:2302.01160}.

\bibitem[{Rosenski, Shamir, and Szlak(2016)}]{rosenski2016multi}
Rosenski, J.; Shamir, O.; and Szlak, L. 2016.
\newblock Multi-player bandits--a musical chairs approach.
\newblock In \emph{International Conference on Machine Learning}, 155--163. PMLR.

\bibitem[{Shi et~al.(2021{\natexlab{a}})Shi, Xiong, Shen, and Yang}]{shi2021heterogeneous}
Shi, C.; Xiong, W.; Shen, C.; and Yang, J. 2021{\natexlab{a}}.
\newblock Heterogeneous multi-player multi-armed bandits: Closing the gap and generalization.
\newblock \emph{Advances in neural information processing systems}, 34: 22392--22404.

\bibitem[{Shi et~al.(2021{\natexlab{b}})Shi, Xiong, Shen, and Yang}]{Shi2021}
Shi, C.; Xiong, W.; Shen, C.; and Yang, J. 2021{\natexlab{b}}.
\newblock Heterogeneous multi-player multi-armed bandits: Closing the gap and generalization.
\newblock \emph{Advances in Neural Information Processing Systems}, 34.

\bibitem[{Szorenyi et~al.(2013)Szorenyi, Busa-Fekete, Hegedus, Orm{\'a}ndi, Jelasity, and K{\'e}gl}]{szorenyi2013gossip}
Szorenyi, B.; Busa-Fekete, R.; Hegedus, I.; Orm{\'a}ndi, R.; Jelasity, M.; and K{\'e}gl, B. 2013.
\newblock Gossip-based distributed stochastic bandit algorithms.
\newblock In \emph{International conference on machine learning}, 19--27. PMLR.

\bibitem[{Wang et~al.(2020)Wang, Proutiere, Ariu, Jedra, and Russo}]{wang2020optimal}
Wang, P.-A.; Proutiere, A.; Ariu, K.; Jedra, Y.; and Russo, A. 2020.
\newblock Optimal algorithms for multiplayer multi-armed bandits.
\newblock In \emph{International Conference on Artificial Intelligence and Statistics}, 4120--4129. PMLR.

\bibitem[{Yang et~al.(2021)Yang, Chen, Pasteris, Hajiesmaili, Lui, and Towsley}]{yang2021cooperative}
Yang, L.; Chen, Y.-Z.~J.; Pasteris, S.; Hajiesmaili, M.; Lui, J.; and Towsley, D. 2021.
\newblock Cooperative stochastic bandits with asynchronous agents and constrained feedback.
\newblock \emph{Advances in Neural Information Processing Systems}, 34: 8885--8897.

\bibitem[{Yang et~al.(2023)Yang, Wang, Hajiesmaili, Zhang, Lui, and Towsley}]{yang2023cooperative}
Yang, L.; Wang, X.; Hajiesmaili, M.; Zhang, L.; Lui, J.; and Towsley, D. 2023.
\newblock Cooperative multi-agent bandits: Distributed algorithms with optimal individual regret and constant communication costs.
\newblock \emph{arXiv preprint arXiv:2308.04314}.

\end{thebibliography}



\end{document}